\documentclass{article}

 \usepackage[preprint]{neurips_2026}

\usepackage[utf8]{inputenc} % allow utf-8 input
\usepackage[T1]{fontenc}    % use 8-bit T1 fonts
\usepackage{hyperref}       % hyperlinks
\usepackage{url}            % simple URL typesetting
\usepackage{booktabs}       % professional-quality tables
\usepackage{amsfonts}       % blackboard math symbols
\usepackage{nicefrac}       % compact symbols for 1/2, etc.
\usepackage{microtype}      % microtypography
\usepackage{xcolor}         % colors
\usepackage{enumitem}
\usepackage{amsmath}
\usepackage{amsthm}
\usepackage{bm}
\usepackage[table]{xcolor}
\usepackage{tcolorbox}

\newtheorem*{proposition*}{Proposition}

\newtheorem{lemma}{Lemma}
\newtheorem{proposition}{Proposition}
\newtheorem{corollary}{Corollary}

\usepackage{amssymb}
\usepackage{mathtools}
\usepackage{caption}  
\usepackage{array}    
\usepackage{algorithm}
\usepackage{algorithmic}
\usepackage{float}             
\usepackage{placeins} 
\usepackage{xcolor} 
\usepackage{setspace}
\usepackage{multirow}
\usepackage{graphicx}
\usepackage{adjustbox}
\usepackage{array}
\usepackage{tabularx}
\usepackage[misc]{ifsym}
\usepackage{natbib}
\setcitestyle{square}

\renewcommand{\cite}{\citep}
\usepackage{wrapfig}
\usepackage{nicefrac}

\title{
\Large
The Hidden Power of Scaling Factor in LoRA Optimization
}

\author{
 Zicheng Zhang$^1$\thanks{Correspondence:  zhangzicheng6@jd.com}
\quad\  {Haoran Li}$^2$
\quad\  \textbf{Jiaxing Wang}$^1$
\quad\   {Guoqiang Gong}$^1$
\quad\   \textbf{Anqi Li}$^3$
\vspace{0.3em}
\\
\quad\   \textbf{Yudong Hu}$^1$
\quad\   \textbf{Ting Xiong}$^1$
\quad\   \textbf{Yurong Gao}$^4$
\quad\   \textbf{Junxing Hu}$^1$
\quad\   \textbf{Zhida Jiang}$^1$
\vspace{0.3em}
\\
\quad\   \textbf{Yifeng Zhang}$^1$
\quad\  \textbf{Pengzhang Liu$^1$}
\quad\  \textbf{Qixia Jiang$^1$}
\vspace{0.5em} 
\\
{$^1$JD}\quad{$^2$School of Mathematical Sciences, UCAS}\\
{$^3$School of Mathematical Sciences, NKU} \\
{$^4$School of Advanced Interdisciplinary Sciences, UCAS}
}

\begin{document}

\maketitle

\begin{abstract}

In Low-Rank Adaptation (LoRA), the scaling factor $\alpha$ is often treated as a mere complement to the learning rate, yet its role in optimization remains poorly understood.
In this paper, we reveal that {the scaling factor $\alpha$ and the learning rate function differently}, with $\alpha$ emerging as the dominant driver of effective optimization, delivering gains that cannot be replicated by learning rate scaling alone.
Through the synergy of extensive empirical analysis and a theoretical Signal-Drift framework, we uncover three findings into LoRA’s scaling mechanism: First, LoRA’s spectral suppression smooths the optimization landscape, rendering standard hyperparameters overly conservative and creating an optimization gap. Second, when leveraging this smoothness to accelerate convergence, $\alpha$ outperforms the learning rate by amplifying the task signal without increasing the drift ratio. Third, the optimal scaling factor follows a sublinear relationship with the rank, well characterized by a square-root law with an unexpectedly large coefficient, revealing the insufficient scaling of existing rank-tied heuristics.
Based on these insights, we propose LoRA-$\alpha$, a minimalist framework that restores $\alpha$ to its principled regime, making LoRA compatible with standard small learning rates. Extensive evaluations across diverse tasks demonstrate that LoRA-$\alpha$ consistently improves performance while streamlining hyperparameter search,  unleashing the learning potential of LoRA.

\end{abstract}

\section{Introduction}

The rapid growth of large pretrained models~\cite{Brown2020LanguageMA,OpenAI2023GPT4TR,touvron2023llama,bai2023qwen,guo2025deepseek} has made efficient adaptation a central challenge, driving the development of Parameter-Efficient Fine-Tuning (PEFT) methods~\cite{peft,Lester2021ThePO,He2021TowardsAU,Edalati2022KronAPE,zhang2025parameter}. Among these, Low-Rank Adaptation (LoRA)~\cite{Hu2021LoRALA} has emerged as a dominant approach due to its efficiency and stability.
LoRA parameterizes weight updates as $\Delta W = \frac{\alpha}{r}BA$, where low-rank factors $B$ and $A$ approximate the update with rank $r$ and scaling factor $\alpha$. This simple formulation reinforced by framework support~\cite{peft}, has enabled widespread adoption from Natural Language Processing (NLP)~\cite{liu2022few,Ding2022DeltaTA,zhao2024lora} to multimodal generation~\cite{guo2023animatediff,blattmann2023stable,ruiz2023dreambooth}.

Despite its conceptual simplicity, the optimization behavior of LoRA remains poorly understood due to its inherent bilinear architecture and the complex interplay of hyperparameters. The scaling factor $\alpha$, originating from feature learning principles~\cite{yang2020feature}, aims to decouple the optimal hyperparameters from the rank selection, thereby streamlining hyperparameter search.
While research has  explored initialization~\cite{lora-ga,PiSSA,zhang2025primacy} and learning rates~\cite{biderman2024lora,schulman2025lora,chen2026learning}, the scaling factor $\alpha$ remains systematically underexplored, typically tethered to simplistic rank-based heuristics, such as $\alpha = r$~\cite{Hu2021LoRALA} or $2r$~\cite{biderman2024lora}.
Consequently, $\alpha$ is often regarded as a secondary alternative to the learning rate for scaling updates, obscuring its pivotal role in reducing hyperparameter search and, more fundamentally, in shaping the underlying optimization regime.

In this work, we analyze the scaling mechanism of LoRA through a joint empirical and theoretical study. 
Across extensive hyperparameter sweeps, we consistently observe that effective optimization hinges more on a sufficiently large scaling factor $\alpha$ than on elevating the learning rate $\eta$. 
To understand this phenomenon, we develop a Signal-Drift framework that delineates LoRA's  optimization characteristics. 
In this view, the task-relevant signal and the bilinear-induced drift respond differently to $\alpha$ and $\eta$, providing a principled explanation for our empirical observations. 
Together, these empirical and theoretical evidence lead to the following three key findings:
\begin{itemize}[leftmargin=*]
    \item \textbf{Spectral Suppression Causes a Scaling Misalignment.}
    We find that LoRA’s low-rank parameterization induces a spectral suppression of the task Hessian, effectively smoothing the optimization landscape. 
    While this improves stability, it also renders standard hyperparameters overly conservative, leading to a significant optimization gap and motivating the aggressive scaling practices observed in prior work~\cite{lora+,schulman2025lora,zhang2025primacy}.

    \item \textbf{$\alpha$ and $\eta$ Play Fundamentally Different Roles.}
    Empirically, we observe that increasing $\alpha$ consistently yields better convergence than increasing $\eta$. 
    Our framework explains this by showing that $\alpha$ amplifies the task-aligned signal, whereas $\eta$ amplifies both the signal and bilinear drift. 
    As a result, $\alpha$ acts as a purity-preserving accelerator, enabling faster and more stable optimization.

    \item \textbf{Optimal Scaling Follows a Sublinear Law.}
    We identify a sublinear relationship between the optimal $\alpha$ and rank $r$, concisely characterized by a square-root law with a large scaling coefficient. 
    This reveals that commonly used rank-tied heuristics~\cite{Hu2021LoRALA,biderman2024lora} operate in a severely under-scaled regime. 
    With proper scaling, LoRA can directly adopt the standard small learning rates used in Full Fine-Tuning (FFT), while achieving superior performance.
\end{itemize}

Based on both empirical and theoretical findings, we identify a key limitation in prevailing LoRA practices: commonly used heuristics restrict $\alpha$ to insufficient magnitudes, limiting LoRA’s optimization capacity. 
To address this, we propose LoRA-$\alpha$, which introduces a large base value with a square-root scaling law to restore $\alpha$ to a principled regime. 
This formulation enables practitioners to bypass costly hyperparameter tuning and directly adopt standard FFT learning rates. 
Extensive experiments across model scales (184M–12B), task domains (natural language, reasoning, multimodal), and training paradigms (supervised, contrastive, and reinforcement learning) show that larger $\alpha$ with standard small $\eta$ consistently yields superior performance. In summary, our contributions include:

\begin{itemize}[leftmargin=*, topsep=0em, itemsep=0.2em]
    \item Empirically, we establish three key findings regarding LoRA's optimization through extensive hyperparameter sweeps. We uncover that performance hinges critically on a large scaling factor, exposing a pitfall in current rank-tied heuristics that leave LoRA's fitting potential underutilized.
    \item Theoretically, we develop a Signal-Drift framework to provide a principled explanation of these findings. By characterizing Hessian spectral suppression, we identify $\alpha$ as a better accelerator and derive a square-root law that aligns LoRA's optimization regime with that of FFT.
    
    \item We propose {LoRA-$\alpha$}, a minimalist framework that elevates the scaling factor to its principled regime while adopting standard FFT learning rates. Across diverse models, tasks, and training paradigms, it significantly improves over LoRA, often reaching performance comparable to FFT.
\end{itemize}

\begin{figure}[t]

    \centering
\includegraphics[width=1.0\linewidth]{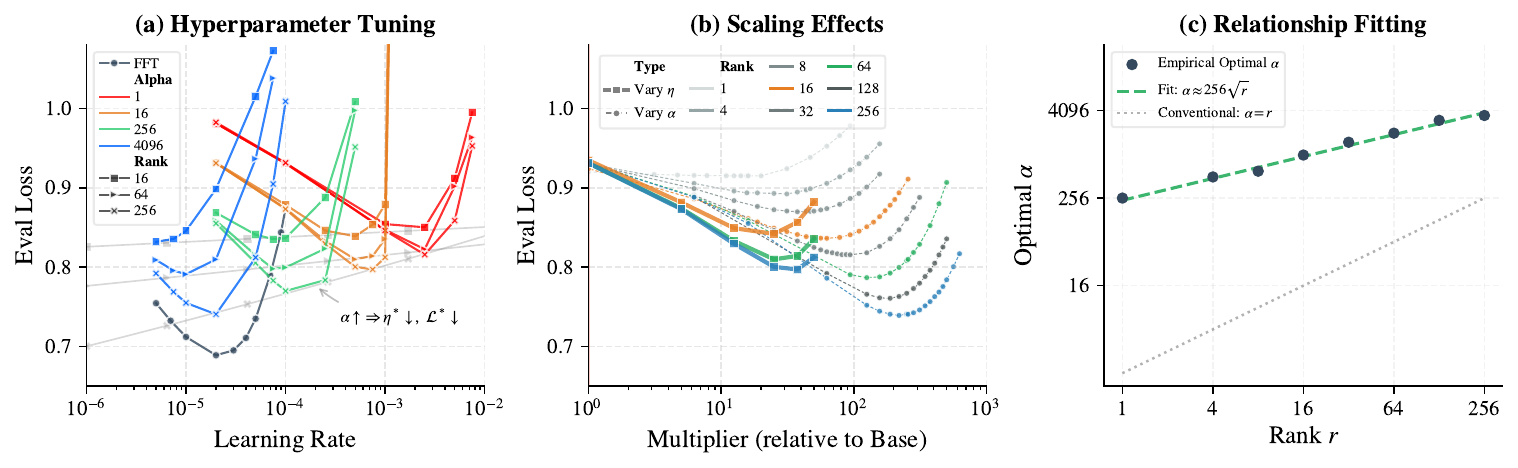}
\caption{Hyperparameter analysis of Llama 3-1B on the Tulu 3 dataset. 
(a) Evaluation loss as a function of the learning rate $\eta$ across different ranks $r$ and scaling factors $\alpha$. Gray lines denote linear fits of the minimum loss for each $(r, \alpha)$. Increasing $\alpha$ lowers the optimal loss and shifts $\eta^*$ downward.
(b) Scaling paths defined in Eq.~\eqref{eq:scaling_path} relative to the baseline ($\eta_{\text{FFT}} = 2\times10^{-5}$, $\alpha_0 = 16$), where $\eta_{\text{FFT}}$ denotes the optimal FFT learning rate. Varying $\alpha$ reaches lower-loss regimes that are inaccessible via $\eta$-tuning alone. 
(c) Optimal scaling factor $\alpha^*$ as a function of rank $r$, evaluated at $\eta_{\text{FFT}}$. The observed sublinear trend and large magnitude of $\alpha$ challenge conventional scaling heuristics.
}
    \label{fig:llama3_tulu3}
\end{figure}

\section{An Empirical Study of LoRA Scaling}
\label{sec:empirical_validation}

To study the scaling behavior of LoRA, we perform systematic hyperparameter sweeps to characterize the interplay between the scaling factor $\alpha$, learning rate $\eta$, and rank $r$. We focus on \textit{fitting behavior}, isolating optimization dynamics from generalization effects such as implicit low-rank regularization~\cite{jang2024lora} and task diversity. 
This enables a clean analysis of the optimization landscape, while generalization is evaluated separately in Section~\ref{sec:experiments}.

\textbf{Optimization Setup and Metrics.}
Following \citet{schulman2025lora}, we conduct Supervised Fine-Tuning (SFT) on two model scales (Llama 3-1B and 8B~\cite{llama3}) and datasets (Tulu 3~\cite{lambert2024tulu} and OpenThoughts~\cite{guha2025openthoughts}), training each configuration for one epoch with a batch size of 32 using AdamW. 
For clarity, we report results on Llama 3-1B with Tulu 3, deferring the rest to Appendix~\ref{appendix:more_fitting}. 
To manage the large search space over $\alpha$, $\eta$, and $r$, we use a 100k-example subset of Tulu 3 with a maximum token length $T=1024$, and reserve 10k examples as a proxy evaluation set $\mathcal{D}_{prox}$.
Let $\theta(r, \eta, \alpha)$ denote the trained parameters under a given configuration. 
We evaluate fitting via the Expected Negative Log-Likelihood (NLL):
\begin{equation}
\mathcal{L}(r, \eta, \alpha) = 
\mathbb{E}_{x \sim \mathcal{D}_{prox}}  \left[-\Sigma _{t=1}^{T} \log P(x_t \mid x_{<t}; \theta(r, \eta, \alpha)) \right].
\end{equation}

\textbf{Interplay between $\alpha$ and $\eta$.} Let $\eta^*(r, \alpha) = \arg\min_{\eta} \mathcal{L}(r, \eta, \alpha)$ denote the optimal learning rate and $\mathcal{L}^*(r, \alpha) = \mathcal{L}(r, \eta^*, \alpha)$ the corresponding minimum loss. As illustrated in Fig.~\ref{fig:llama3_tulu3}(a), we confirm the observation of \citet{schulman2025lora} that LoRA requires roughly $10\times$ larger learning rates than FFT at small $\alpha$ (\textit{e.g.}, $\alpha=16$). 
Nonetheless, our extended sweeps reveal a clear coupling: as $\alpha$ increases, $\eta^*$ consistently decreases, while the optimal loss $\mathcal{L}^*$ improves monotonically. 
This gain becomes more pronounced at higher ranks (gray lines in Fig.~\ref{fig:llama3_tulu3}(a)). \textit{Finding I: The need for large learning rates in LoRA is a symptom of inadequate $\alpha$-scaling, rather than an intrinsic property.}

\textbf{Superiority of $\alpha$-Scaling over $\eta$-Scaling.} To disentangle the roles of the learning rate and scaling factor, we trace two independent scaling paths from the same baseline:
\begin{equation}\label{eq:scaling_path}
\mathcal{L}_{\eta}(m;r,\eta_{\text{FFT}},\alpha_0) = \mathcal{L}(r, m\eta_{\text{FFT}}, \alpha_0), \quad
\mathcal{L}_{\alpha}(m;r,\eta_{\text{FFT}},\alpha_0) = \mathcal{L}(r, \eta_{\text{FFT}}, m\alpha_0),
\end{equation}
where $\eta_{\text{FFT}} = 2\times 10^{-5}$ is the optimal FFT learning rate, and $\alpha_0 = 16$.
As illustrated in Fig.~\ref{fig:llama3_tulu3}(b), although both trajectories exhibit U-shaped loss surfaces, they reveal a clear asymmetry. First, $\alpha$-scaling accommodates a substantially larger optimal multiplier, indicating a more stable optimization regime. Second, while $\eta$-scaling prematurely plateaus, $\alpha$-scaling consistently converges to deeper loss minima, corroborating the broad trends in Fig.~\ref{fig:llama3_tulu3}(a).
\textit{Finding II: $\alpha$ is not merely an alternative to $\eta$ for scaling step sizes; instead, it reshapes the optimization landscape to facilitate deeper fitting.}

\textbf{Scaling Law between $\alpha$ and $r$.}
We examine the relationship between the rank and the optimal scaling factor, defined as $\alpha^*(r) = \arg\min_{\alpha} \mathcal{L}(r, \eta_{\text{FFT}}, \alpha)$. 
 Fig.~\ref{fig:llama3_tulu3}(c) shows that $\alpha^*$ follows a sublinear trend:
\begin{equation}\label{eq:fitting_law}
\alpha^*(r) \approx C \sqrt{r}, \quad C \gg 1,
\end{equation}
with $C$ obtained and rounded via log-scale fitting\footnote{Estimated via linear regression in log space, \textit{i.e.,} $\log\alpha=0.5\log r + \log {C}$}. 
Across all four diverse optimization settings, the optimal $\alpha$ consistently lies in a high-magnitude regime ($\alpha \ge 256$) under typical FFT learning rates.
This indicates that conventional choices are not only misaligned in functional form but also severely under-scaled in magnitude.
 \textit{Finding III: Conventional scaling is mis-scaled, both in its linear dependence on $r$ and its insufficient magnitude, thereby limiting LoRA’s fitting effectiveness.}

\textbf{Relationship to Prior Works.}
Our findings unify and extend prior works on LoRA scaling. 
The original formulation of \citet{Hu2021LoRALA} adopts $\alpha=r$ as a simplifying default, while \citet{biderman2024lora} observes that increasing $\alpha$ improves performance, albeit only exploring up to $\alpha=2r$. 
Meanwhile, \citet{kalajdzievski2023rank} derives the mild scaling ($1/\sqrt{r}$) for stability, without identifying the appropriate magnitude. 
In addition, \citet{biderman2024lora,schulman2025lora} report that LoRA requires learning rates $10\times$ larger than FFT. Our results both confirm this observation and attribute it to insufficient scaling of $\alpha$. 
By determining the concrete scaling law, our work consolidates these perspectives into a unified principle, enabling superior optimization with standard learning rates and substantially reducing the need for hyperparameter search.

\section{A Theoretical Analysis from the Signal-Drift Perspective}
\label{sec:theory}

In this section, we formally characterize how LoRA reshapes optimization regime to ground our empirical observations. Our primary analytical approach is to introduce the Signal-Drift decomposition, which isolates the structural drift inherently induced by the adapter's bilinear parameterization.

\textbf{Notations.} Consider a pre-trained weight $W_0 \in \mathbb{R}^{d_{out} \times d_{in}}$ and trainable adapters ${\theta} = \{A, B\}$, where $B \in \mathbb{R}^{d_{out} \times r}$ and $A \in \mathbb{R}^{r \times d_{in}}$. The LoRA mapping is defined as $W(\bm{\theta}) = W_0 + \frac{\alpha}{r}BA$. Let $\bm{w}(\bm{\theta}) = \operatorname{vec}(W(\bm{\theta}))\in \mathbb{R}^{D}$ denote the vectorized weight, and $\bm{\theta} =[\operatorname{vec}(A); \operatorname{vec}(B)] \in \mathbb{R}^p$ the concatenated parameters, where $D = d_{out}d_{in}$ and $p = r(d_{in} + d_{out})$.
For the objective $\ell(\bm{w})$, we define the task gradient $\bm{g} = \nabla_{\bm{w}} \ell \in \mathbb{R}^D$ and the task Hessian $\mathbf{H}_\ell = \nabla^2_{\bm{w}} \ell \in \mathbb{R}^{D \times D}$, with $g_k$ and $w_k$ representing their $k$-th scalar entries. The mapping geometry is governed by the Jacobian $J(\bm{\theta}) = \nabla_{\bm{\theta}} \bm{w} \in \mathbb{R}^{D \times p}$ and the structural Hessian $\nabla^2_{\bm{\theta}} \bm{w} \in \mathbb{R}^{D \times p \times p}$, interpreted as a bilinear operator: for any $\bm{v} \in \mathbb{R}^p$, $\nabla^2_{\bm{\theta}} \bm{w}[\bm{v}, \bm{v}] \in \mathbb{R}^D$ has entries $\bm{v}^\top (\nabla^2_{\bm{\theta}} w_k)\bm{v}$.

\subsection{Decomposing the Optimization Dynamics}
\label{subsec:static_decomp}
By parameterizing the objective as $\ell(\bm{w}(\bm{\theta}))$, LoRA optimization inevitably inherits artifacts from its bilinear architecture. Specifically, the dynamics couple with the mapping's intrinsic structure $\nabla^2_{\bm{\theta}} \bm{w}$, prompting us to decouple the task-aligned projection signal from this task-agnostic structure {drift}.

\begin{figure}[t]
    \centering
    \includegraphics[width=1.0\linewidth]{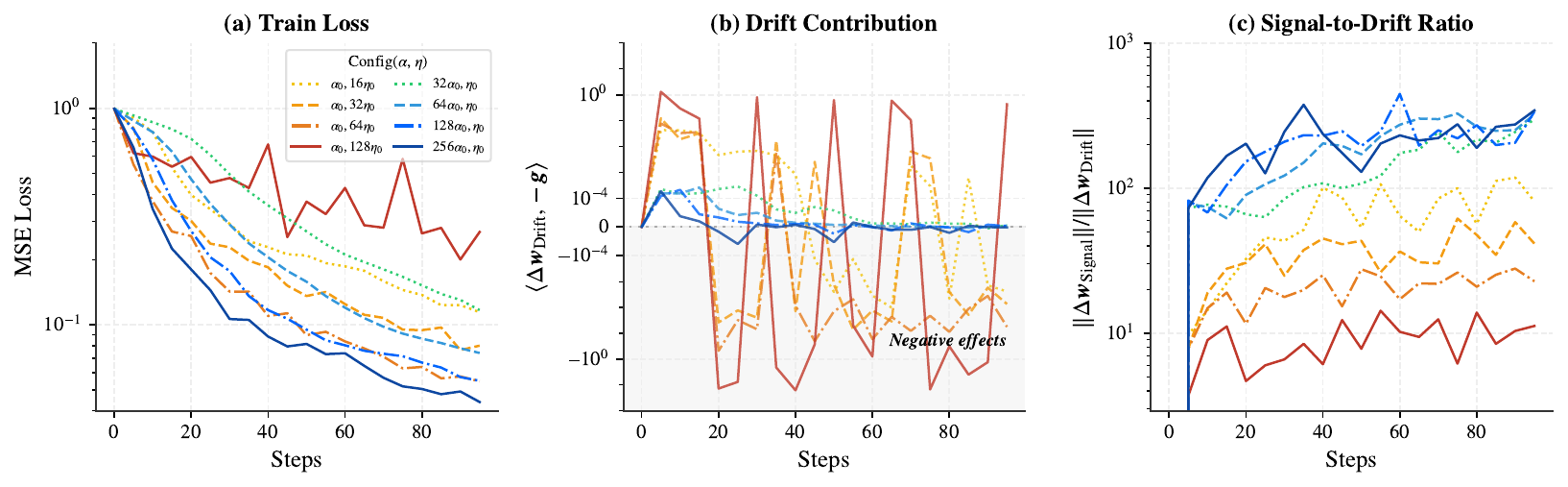}
\caption{
    The asymmetric optimization dynamics of $\alpha$-scaling versus $\eta$-scaling from a base configuration of rank $r=16$, $\alpha_0=1$, and $\eta_0=10^{-4}$. 
    We compare increasing $\eta$ (warm colors) versus increasing $\alpha$ (cool colors).
    (a) $\eta$-scaling leads to early saturation and instability, while $\alpha$-scaling enables smooth acceleration and deeper fitting.
    (b)   Increasing $\eta$ amplifies the pronounced stochasticity of structural drift, whereas $\alpha$-scaling does not further amplify these fluctuations, maintaining the smoother optimization profile seen in (a).
    (c) $\alpha$-scaling consistently maintains a higher Signal-to-Drift ratio, preserving update purity and accelerating convergence.  
}
  \label{fig:sdr_convergence}
\end{figure}

\begin{proposition}[Signal-Drift Decomposition]
\label{prop:sd_decomp}
Let $\Delta \bm{w}_{\text{LoRA}} = \bm{w}(\bm{\theta} + \Delta\bm{\theta}) - \bm{w}(\bm{\theta})$ denote the effective step in the weight space, and $\mathcal{H}_{\text{LoRA}} = \nabla^2_{\bm{\theta}} \ell$ denote the parameter-space Hessian. Under the LoRA parameterization, both admit an exact decomposition into a task-aligned signal and a structural drift:
\begin{align}
\Delta \bm{w}_{\text{LoRA}}
= \underbrace{J(\theta)\Delta \bm{\theta}}_{\Delta \bm{w}_{\text{Signal}}}
+ \underbrace{\tfrac{1}{2}\nabla^2_{\bm{\theta}} \bm{w}[\Delta \bm{\theta}, \Delta \bm{\theta}]}_{\Delta \bm{w}_{\text{Drift}}}, \quad
\mathcal{H}_{\text{LoRA}} 
= \underbrace{J(\theta)^\top \mathbf{H}_\ell J(\theta)}_{\mathcal{H}_{\text{Signal}}}
+ \underbrace{\Sigma_{k=1}^{D} g_k \nabla^2_{\bm{\theta}} w_k}_{\mathcal{H}_{\text{Drift}}}.
\end{align}
\end{proposition}
Algebraically, the components take the form $\Delta \bm{w}_{\text{Signal}} = \operatorname{{vec}}\big(\frac{\alpha}{r} (B \Delta A + \Delta B A)\big)$ and the bilinear drift $\Delta \bm{w}_{\text{Drift}} = \operatorname{{vec}}\big(\frac{\alpha}{r} \Delta B \Delta A\big)$, where $\Delta A$ and $\Delta B$ denote the  updates within $\Delta {\theta}$.

\begin{proposition}[Geometric Properties of Signal and Drift]
\label{prop:geometric_properties}
For any gradient-based update step\footnote{This conceptually extends to adaptive optimizers (\textit{e.g.}, Adam, Muon) operating via preconditioned gradients.} $\Delta \bm{\theta} = -\eta \nabla_{\bm{\theta}} \ell$, the decomposed components in Proposition~\ref{prop:sd_decomp} satisfy: 

\textbf{(i) Constructive Signal:} The signal component  aligns with the  descent direction, $\langle \Delta \bm{w}_{\text{Signal}}, -\bm{g} \rangle \ge 0$, and preserves the local convexity of the loss landscape, \textit{i.e.}, $\mathcal{H}_{\text{Signal}}  \succeq 0$ given $\mathbf{H}_\ell \succeq 0$.

\textbf{(ii) Adversarial Drift:} The structural Hessian $\mathcal{H}_{\text{Drift}}$ is strictly indefinite, and its induced update $\Delta \bm{w}_{\text{Drift}}$ exerts an uncontrolled force with no guaranteed alignment with $-\bm{g}$.
\end{proposition}

The signal guarantees descent aligned with the task objective, whereas the drift introduces task-agnostic instability from the bilinear parameterization. Detailed proofs are deferred to Appendix~\ref{appendix:proof}.

\subsection{Theoretical Justifications for Empirical Findings}
\label{subsec:theory_results}
Based on the Signal-Drift decomposition, we theoretically justify the empirical phenomena observed in Section~\ref{sec:empirical_validation}. For intuition, we utilize a simple MLP for visualization, circumventing the intractable $\mathcal{O}(D^2)$ complexity of tracking $\mathbf{H}_\ell$ and $\nabla^2_{\bm{\theta}}\bm{w}$ in large models. See Appendix~\ref{appendix:toy_example_details} for details.

\textbf{Justification for Finding I: Spectral Suppression Enables Aggressive Hyperparameters.}
Empirically, LoRA tolerates much larger learning rates than FFT without diverging. This enhanced stability stems from the inherent spectral suppression of the LoRA optimization landscape.

\begin{proposition}[Spectral Suppression]
\label{prop:suppression}
Given the initial state of LoRA where $B=0$ and $A_{i,j} \sim \mathcal{N}(0, \sigma^2_{A})$, the expected signal curvature satisfies $\mathbb{E}[\operatorname{Tr}(\mathcal{H}_{\text{Signal}})] = \alpha^2\rho \operatorname{Tr}(\mathbf{H}_{\ell})$, where $\rho = \sigma^2_{A}/r$.
\end{proposition}
Standard initialization like $\sigma_A^2 = 1/d_{in}$ yields $\rho \ll 1$, substantially compressing its maximum eigenvalue $\lambda_{\max}[\mathcal{H}_{\text{Signal}}]$ relative to $\mathbf{H}_{\ell}$. According to optimization theory~\cite{nesterov2013introductory}, this reduced maximum eigenvalue lowers the gradient Lipschitz constant, thereby expanding the stable learning rate bound that $\eta \le 2/\lambda_{\max}[\mathcal{H}_{\text{Signal}}]$. Consequently, spectral suppression smooths the landscape, rendering hyperparameters from FFT overly conservative and suboptimal for LoRA. This behavior is derived in Appendix Lemma~\ref{lemma:stability_bound} and further evidenced by spectral analysis in Fig.~\ref{fig:spectral_properties}(a-b).

\textbf{Justification for Finding II: The Asymmetric Roles of Scaling Factor and Learning Rate.}
To explain why $\alpha$-scaling achieves superior fitting capacity compared to $\eta$-scaling, our decomposition reveals a fundamental asymmetry between the two hyperparameters.

\begin{proposition}[Asymmetric Scaling]
\label{prop:asymmetric_scaling}
The components of the landscape and the weight update scale asymmetrically with respect to $\alpha$ and $\eta$. 
For the landscape: $\mathcal{H}_{\text{Signal}} = \Theta(\alpha^2)$ while $\mathcal{H}_{\text{Drift}} = \Theta(\alpha)$. 
For the update under adaptive optimizers like Adam: $\Delta \bm{w}_{\text{Signal}} = \Theta(\alpha\eta)$ while $\Delta \bm{w}_{\text{Drift}} = \Theta(\alpha\eta^2)$.
\end{proposition}

This asymmetry dictates their distinct optimization roles. Specifically, increasing $\eta$ strictly degrades the Signal-to-Drift Ratio, defined as $\|\Delta \bm{w}_{\text{Signal}}\| / \|\Delta \bm{w}_{\text{Drift}}\| = \Theta(1/\eta)$. This  degradation manifests empirically as the drift-induced volatility and premature saturation observed in Fig.~\ref{fig:sdr_convergence}.  Conversely, $\alpha$-scaling enables spectral purification. As shown in Fig.~\ref{fig:spectral_properties}(c-d), increasing $\alpha$ reshapes the LoRA landscape $\mathcal{H}_{\text{LoRA}} $ to recover the spectral profile of signal term.

\textbf{Justification for Finding III: Conventional Scaling is Fundamentally Mis-scaled.}
Empirically, the optimal $\alpha$ is exceptionally large and scales sublinearly with rank $r$. This follows from Proposition~\ref{prop:suppression}: to counteract spectral suppression and restore expressiveness, the adapter's expected signal curvature must match the FFT curvature scale.

\begin{corollary}[Initial Curvature Alignment]
\label{cor:optimal_scaling}
To match the signal curvature to the full task curvature at initialization, \textit{i.e.}, $\mathbb{E} [\operatorname{Tr}(\mathcal{H}_{\text{Signal}})] = \Theta\big(\operatorname{Tr}(\mathbf{H}_{\ell})\big)$, the scaling factor follows  $\alpha = \Theta\big(\sqrt{{r}/{\sigma_A^2}}\big)$.
\end{corollary}

This derivation proves that conventional linear scaling is mis-specified. Instead, $\alpha$ should scale sublinearly with rank and grow substantially with model width under standard initialization.

\section{The LoRA-$\alpha$ Protocol}
\label{sec:methodology}
Our empirical and theoretical analyses converge on a central principle: 
\textit{effective LoRA tuning requires a properly calibrated scaling factor}. 
By aligning the signal curvature of LoRA with that of FFT, such calibration provides two key benefits: 
(i) it enables the direct use of the standard FFT learning rate ($\eta_{\text{FFT}}$) without costly hyperparameter sweeps, and 
(ii) it unlocks stronger fitting capacity than learning-rate scaling alone. Based on this insight, we propose {LoRA-$\alpha$} to redefine LoRA as:
\begin{equation}
    \Delta W = \frac{\alpha \cdot \alpha_{\text{base}}}{r} BA, 
    \quad \text{with} \ \eta = \Theta(\eta_{\text{FFT}}).
\end{equation}

This decouples the roles of scaling and optimization: 
$\alpha_{\text{base}}$ restores the suppressed curvature induced by LoRA, 
while $\alpha$ serves as a lightweight tuning knob that can be safely fixed to $1$ by default. 
If further adjustment is needed, $\alpha$ can be searched within a narrow and intuitive range 
(\textit{e.g.}, $\alpha \in [0.1, 10]$), eliminating the need for expensive joint tuning over scaling factors and learning rates. We provide two complementary strategies for determining $\alpha_{\text{base}}$:

\textbf{Variant I: Empirical Scaling for LLMs.} Modern LLMs exhibit strong architectural homogeneity, with hidden dimensions typically on the order of thousands (\textit{e.g.}, $d_{in} = 4096$), allowing for a universal hyperparameter regime. Leveraging the scaling law from Section~\ref{sec:empirical_validation}, we adopt $C=256$ as a robust default, yielding: $\alpha_{\text{base}} = 256\sqrt{r}$.

\textbf{Variant II: Analytic Scaling.}
For general architectures, we derive a principled scaling from Corollary~\ref{cor:optimal_scaling} by matching the expected task curvature at initialization:
$\alpha_{\text{base}} = \frac{1}{\sigma_A}\sqrt{r}$.
This closed-form solution provides a theoretical baseline that aligns LoRA with the hyperparameter regime of FFT.

Importantly, Variant I serves as a \textit{global empirical heuristic}, while Variant II provides a \textit{layer-wise theoretical initialization}. 
In practice, the empirical scaling in Variant I  yields marginally larger magnitudes than the analytic prediction of Variant II for modern LLMs.  A detailed magnitude comparison is provided in Appendix~\ref{appendix:Comparison_Scaling_Factor}. We identify this gap as an important empirical finding. 
A plausible explanation is that the analytic scaling matches the {average} curvature via the trace of the Hessian, whereas optimization stability is governed by the {largest} eigenvalue. As shown in Fig.~\ref{fig:spectral_properties}(a-b), these quantities can differ within a moderate range, making such discrepancies reasonable.
Ultimately, this principled choice resolves the $16\times$--$256\times$ under-scaling inherent in prevailing heuristics.

\section{Empirical Validation of LoRA-$\alpha$}
\label{sec:experiments}

To rigorously evaluate the effectiveness  of LoRA-$\alpha$, we conduct extensive experiments along three dimensions: (i) Model diversity, covering encoder, decoder, diffusion, and vision-language architectures; (ii) Task heterogeneity, including natural language and multimodal tasks, encompassing cross entropy, reward-based, flow matching and contrastive learning objectives; (iii)  Training horizon, ranging from few-shot adaptation to long-horizon optimization. This comprehensive design validates LoRA-$\alpha$ across both adaptation benchmarks and post-training scenarios. Specifically, we verify the effectiveness of \textit{Analytic Scaling} on diverse architectures in Section~\ref{section:various_tasks}, while adopting \textit{Empirical Scaling} as the default configuration for the remaining LLM-centric experiments. All experiments are implemented using the PEFT library~\cite{peft} and executed on eight NVIDIA H800 GPUs. Due to space limitations, we briefly present the core settings and main results in the following subsections. Comprehensive implementation details are deferred to Appendix~\ref{appendix:Implementation}.

\begin{table*}[t]
\centering
\caption{Comparison on the GLUE benchmark. 
Experiments are conducted using the DeBERTa-v3-base-184M model with a base learning rate of $1\times 10^{-4}$ and a rank of $8$.
Asterisk ($*$) denotes optimal results from a $0.1\times$--$10\times$ sweep over $\eta$ or $\alpha$. Best and second-best values are bold and underlined.}
\label{tab:NLU_comparison_final}
\small
\renewcommand{\arraystretch}{0.95}
\setlength{\tabcolsep}{4.5pt}
\begin{tabular*}{\textwidth}{@{\extracolsep{\fill}}lccccccccc}
\toprule
\textbf{Method} & \textbf{MNLI} & \textbf{SST-2} & \textbf{MRPC} & \textbf{CoLA} & \textbf{QNLI} & \textbf{QQP} & \textbf{RTE} & \textbf{STS-B} & \textbf{Avg.} \\ \midrule
Full Fine-Tuning & 88.31 & 93.57 & 89.46 & 67.26 & 92.80 & 91.52 & 83.75 & 86.87 & 86.69 \\ \midrule
LoRA    & 90.23 & \textbf{95.87} & 84.06 & 63.56 & 93.88 & 90.55 & 50.18 & 87.20 & 81.94 \\
RsLoRA  & 90.33 & 95.64 & 86.38 & 64.85 & 93.97 & 90.26 & 60.31 & 88.37 & 83.76 \\
LoRA+   & 90.37 & 95.32 & 87.54 & 64.79 & \textbf{94.32} & 90.93 & 65.37 & 89.20 & 84.73 \\
PiSSA   & \underline{90.38} & 95.64 & \underline{89.21} & 65.06 & 93.84 & 91.35 & 74.36 & 88.90 & 86.09 \\
LoRAM   & 90.34 & 95.29 & \textbf{89.95} & \underline{65.53} & \underline{94.08} & \underline{91.70} & \underline{74.72} & \underline{89.93} & \underline{86.44} \\
{LoRA-$\alpha$} & \textbf{90.43} & \underline{95.75} & \textbf{89.95} & \textbf{70.89} & 94.03 & \textbf{91.74} & \textbf{80.86} & \textbf{90.12} & \textbf{87.97} \\ \midrule
LoRA-$\eta^{*}$ & 90.69 & 96.21 & 91.05 & 69.27 & \textbf{94.05} & 91.76 & 84.47 & 90.72 & 88.53 \\
{LoRA-$\alpha^{*}$} & \textbf{90.75} & \textbf{96.44} & \textbf{90.44} & \textbf{70.89} & \textbf{94.05} & \textbf{92.06} & \textbf{87.00} & \textbf{91.11} & \textbf{89.09} \\
\bottomrule
\end{tabular*}
\end{table*}

\begin{table*}[t]
\centering
\caption{Comparison on NLG tasks using Llama 2-7B tuned with a base learning rate of $2\times10^{-5}$. }
\label{tab:nlg_performance_final}
\small
\renewcommand{\arraystretch}{0.7}
\setlength{\tabcolsep}{6pt}
\begin{tabular*}{\textwidth}{@{\extracolsep{\fill}}clcccccc}
\toprule
\textbf{Rank ($r$)} & \textbf{Method} & \textbf{GSM8K} & \textbf{MATH} & \textbf{HumanEval} & \textbf{MBPP} & \textbf{Commonsense} & \textbf{Avg.} \\
\midrule
N/A & Full Fine-Tuning & 60.34 & 11.74 & 32.30 & 39.27 & 79.20 & 44.57 \\
\midrule
\multirow{8}{*}{16} 
& LoRA & 31.51 & 4.16 & 15.98 & \underline{28.65} & 66.56 & 29.37 \\
& RsLoRA & {39.04} & 4.94 & {18.85} & 28.10 & 73.24 & 32.83 \\
& LoRA+ & 31.69 & 3.98 & 18.54 & 28.00 & 72.19 & 30.88 \\
& PiSSA & 37.68 & {5.16} & 18.37 & 28.62 & {73.72} & 32.71 \\
& LoRAM & \underline{40.32} & \underline{5.30} & \textbf{18.92} & \textbf{28.83} & \underline{75.19} & \underline{33.71} \\
& LoRA-$\alpha$ & \textbf{47.61} & \textbf{6.84} & \underline{18.90} & 28.30 & \textbf{77.42} & \textbf{35.81} \\
\cmidrule{2-8}
& LoRA-$\eta^{*}$ & {54.28} & 9.50 & 24.40 & 34.70 & \textbf{78.19} & 40.21 \\
& {LoRA-$\alpha^{*}$} & \textbf{55.34} & \textbf{9.84} & \textbf{28.70} & \textbf{35.40} & 78.16 & \textbf{41.49} \\
\midrule
\multirow{8}{*}{128}
& LoRA & 40.27 & 4.72 & 20.11 & 28.84 & 73.64 & 33.52 \\
& RsLoRA & 50.38 & \underline{7.32} & 21.32 & 30.73 & 77.01 & 37.35 \\
& LoRA+ & 40.41 & 5.28 & 20.71 & 29.13 & \underline{78.19} & 34.74 \\
& PiSSA & \underline{51.48} & 7.04 & {21.62} & {31.07} & 77.28 & 37.70 \\
& LoRAM & 51.12 & 7.25 & \underline{22.03} & \underline{31.53} & 77.81 & \underline{37.95} \\
& LoRA-$\alpha$  & \textbf{53.15} & \textbf{8.30} & \textbf{23.20} & \textbf{33.90} & \textbf{78.38} & \textbf{39.39} \\
\cmidrule{2-8}
& LoRA-$\eta^{*}$ & 58.83 & 10.76 & 28.70 & 36.50 & 78.67 & 42.69 \\
& LoRA-$\alpha^{*}$ & \textbf{59.97} & \textbf{11.74} & \textbf{33.50} & \textbf{39.90} & \textbf{79.26} & \textbf{44.87} \\
\bottomrule
\end{tabular*}
\end{table*}

\subsection{Performance on Fundamental Adaptation Tasks}
\label{section:various_tasks}

We evaluate short-horizon adaptation within a few hours across both small- and large-scale models. We benchmark our approach against vanilla LoRA~\cite{Hu2021LoRALA} ($\alpha = r$) and representative methods, including scaling-based RsLoRA~\cite{kalajdzievski2023rank}, learning-rate-modified LoRA+~\cite{lora+}, spectral-initialized PiSSA~\cite{PiSSA}, and magnitude-initialized LoRAM~\cite{zhang2025primacy}. To ensure a fair comparison, we adopt the baseline results from~\citet{zhang2025primacy}, independently training our method under identical settings.

\textbf{Natural Language Understanding (NLU).} 
At first, we fine-tune and evaluate {DeBERTa-v3-base-184M}~\cite{he2021debertav3} on eight datasets from the GLUE benchmark~\cite{Wang2018GLUEAM}. 
All methods are trained with a {batch size of 32}, {learning rate of $1\times10^{-4}$}, and {LoRA rank $r=8$} for 3--5 epochs. 
As shown in Table~\ref{tab:NLU_comparison_final}, LoRA-$\alpha$ achieves the strongest or near-strongest performance across tasks, outperforming both vanilla LoRA and alternative scaling strategies in average performance.

\textbf{Natural Language Generation (NLG).}
We then fine-tune {Llama 2-7B}~\cite{touvron2023llama} on three datasets covering mathematical reasoning (MetaMathQA), code generation (CodeFeedback), and commonsense reasoning (Commonsense170K), respectively. 
Training is conducted with a {batch size of 128}, {learning rate of $2\times10^{-5}$}, and {LoRA ranks $r \in \{16,128\}$} for {one epoch over 100k samples}. 
As shown in Table~\ref{tab:nlg_performance_final}, LoRA-$\alpha$ matches or surpasses existing LoRA variants across all tasks and rank settings. The advantage becomes more pronounced at higher rank of 128, indicating that proper $\alpha$ scaling is increasingly critical as model capacity grows.

\textbf{Text-to-Image Synthesis.}
To validate scalability to multimodal generative settings, we fine-tune Flux.1-12B~\cite{flux2024} using the DreamBooth protocol~\cite{ruiz2023dreambooth}. 
All methods are trained with a {batch size of 1}, {learning rate of $1\times10^{-4}$}, and {LoRA rank $r=8$} for {1,000 training steps}, supervised by flow matching loss. 
As illustrated in Fig.~\ref{fig:flux}, LoRA-$\alpha$ produces images with higher fidelity to reference concepts compared to PiSSA and LoRAM, demonstrating that improved scaling generalizes beyond language tasks to complex generative domains.

\textbf{Comparison of $\alpha$-scaling and $\eta$-scaling}.
We perform hyperparameter sweeps over both $\alpha$ and $\eta$ on NLU and NLG tasks, with full results provided  in Appendix \ref{appendix:Implementation}. As reported in Tables~\ref{tab:NLU_comparison_final} and~\ref{tab:nlg_performance_final}, $\alpha$-scaling consistently yields superior performance compared to $\eta$-scaling. Moreover, the performance gap widens significantly at larger ranks, corroborating the initial scaling findings from Fig.~\ref{fig:llama3_tulu3}.

\begin{figure}[t]
    \centering
\includegraphics[width=0.98\linewidth]{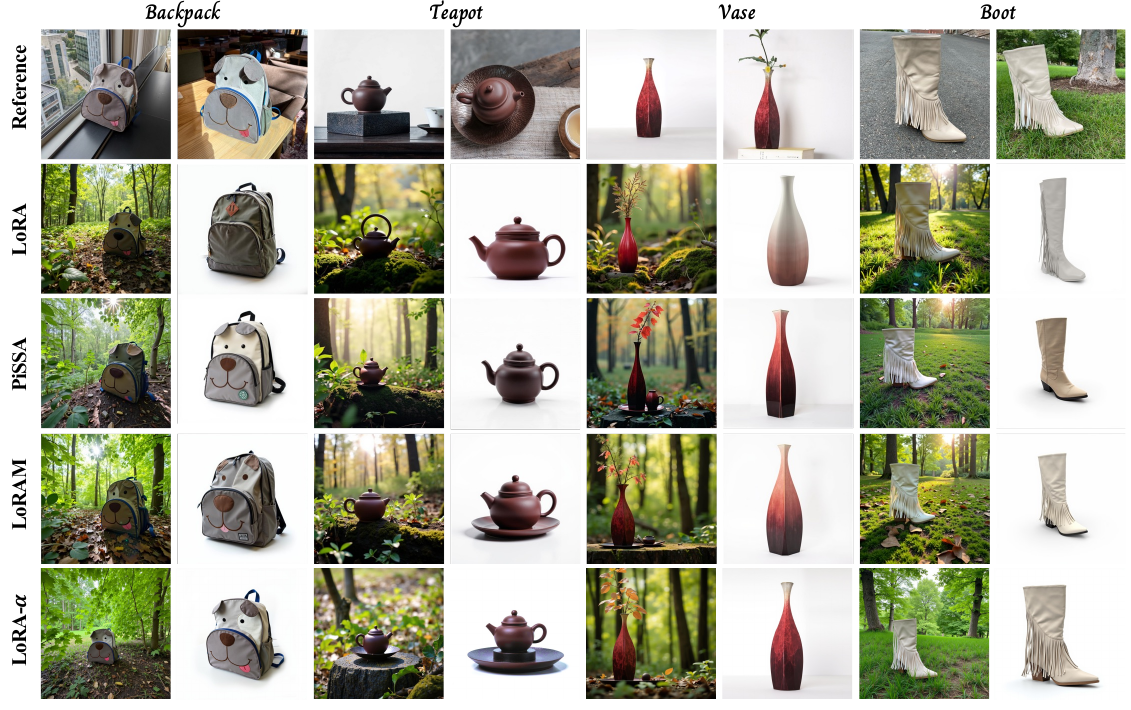}
    \caption{Qualitative comparison of image customization on Flux.1-12B with a base learning rate of $1\times10^{-4}$ and a rank of 8. LoRA-$\alpha$ consistently achieves higher fidelity to reference objects.}
    \label{fig:flux}
\end{figure}

\begin{table*}[t]
\centering
\caption{Comparison on the MMEB benchmark via post-tuning Qwen 2-VL ($r=16$, base $\eta=2\times10^{-5}$). ``VLM2VEC'' denotes officially released results. IND and OOD indicate in- and out-of-distribution splits, respectively, with scores reported to one decimal place per the VLM2Vec protocol. \textbf{Bold} and \underline{underlined} values denote the best and second-best results within each model group.}
\label{tab:embedding_performance_compact}
\small
\renewcommand{\arraystretch}{0.9} 
\setlength{\tabcolsep}{0pt}
\begin{tabular*}{\textwidth}{@{\extracolsep{\fill}}ll cccc c cc c}
\toprule
\multirow{2}{*}{\textbf{Model}} & \multirow{2}{*}{\textbf{Method}} & \multicolumn{4}{c}{\textbf{Meta-Task Average Score}} && \multicolumn{3}{c}{\textbf{Average Score}} \\
\cmidrule(lr){3-6} \cmidrule(lr){8-10}
& & Classification & VQA & Retrieval & Grounding && IND & OOD & Overall \\
\midrule
& \# of datasets $\rightarrow$ & 10 & 10 & 12 & 4 && 20 & 16 & 36 \\
\midrule
\multirow{7}{*}{2B} & VLM2Vec & \underline{59.0} & 49.4 & 65.4 & 73.4 && 66.0 & \underline{52.6} & 60.1 \\
& Full Fine-Tuning & 58.3 & 48.9 & 64.8 & \underline{74.5} && \textbf{67.2} & 50.3 & 59.6 \\
& LoRA & 58.7 & 47.2 & 64.5 & 70.5 && 64.0 & 52.1 & 58.7 \\
& DoRA & 57.3 & 43.2 & 60.8 & 67.0 && 61.3 & 48.6 & 55.6 \\
& PiSSA & 52.5 & 44.2 & 61.4 & 66.1 && 59.3 & 48.9 & 54.7 \\
& LoRA ($10\times$ LR) & \underline{59.0} & \underline{50.9} & \underline{65.8} & 73.0 && \underline{66.9} & 52.5 & \underline{60.5} \\
& {LoRA-$\alpha$} & \textbf{60.2} & \textbf{52.3} & \textbf{66.0} & \textbf{75.2} && \textbf{67.2} & \textbf{54.7} & \textbf{61.6} \\
\midrule
\multirow{6}{*}{7B} & VLM2Vec & \underline{62.6} & 57.8 & \underline{69.9} & \underline{81.7} && \underline{72.7} & \underline{57.8} & \underline{65.8} \\
& LoRA & 60.7 & 57.3 & 68.1 & 81.2 && 70.7 & 56.8 & 64.5 \\
& DoRA & 61.6 & 54.2 & 66.5 & 78.2 && 69.0 & 55.6 & 63.0 \\
& PiSSA & 58.2 & 55.0 & 65.4 & 77.9 && 67.9 & 54.4 & 61.9 \\
& LoRA ($10\times$ LR) & 62.0 & \underline{58.9} & 69.3 & 81.3 && 72.6 & 55.4 & 65.3 \\
& {LoRA-$\alpha$} & \textbf{63.0} & \textbf{59.4} & \textbf{70.6} & \textbf{84.2} && \textbf{73.3} & \textbf{58.8} & \textbf{66.9} \\
\bottomrule
\end{tabular*}
\end{table*}

\subsection{Performance on Multimodal Representation Learning}

Moving beyond short-horizon tasks, we evaluate LoRA-$\alpha$ in a long-horizon representation learning setting by transforming a multimodal LLM into a dense retriever on the MMEB benchmark~\cite{jiang2025vlmvec}. Following VLM2Vec~\cite{jiang2025vlmvec}, we train Qwen 2-VL (2B and 7B) using an InfoNCE objective, with a batch size of 1024,  learning rate of $2\times10^{-5}$, maximum sequence length of 4096 tokens, and LoRA rank of $r=16$, taking $\sim$3 days per run. We compare LoRA-$\alpha$ against {standard LoRA}, LoRA with $10\times$ learning rate~\cite{schulman2025lora}, {DoRA}~\cite{liu2024dora}, {PiSSA}, and  FFT, where FFT is only feasible for the 2B model due to memory constraints.

\textbf{Results.} 
As shown in Table~\ref{tab:embedding_performance_compact}, LoRA-$\alpha$ consistently outperforms all PEFT baselines and FFT across both model scales. Notably, it surpasses the best checkpoint from official VLM2Vec, indicating that improved scaling alone can unlock additional performance without architectural modifications.
The largest gains are observed on OOD evaluation, where LoRA-$\alpha$ significantly improves over FFT and VLM2Vec, demonstrating stronger robustness under distribution shifts.

\begin{table*}[t]
\centering
\caption{Comparison of Pass@1 accuracy across reasoning benchmarks. We train Qwen 2.5-Math-7B with a base learning rate of $4\times 10^{-5}$. $n$ denotes the number of samples used for Pass@1 estimation.}
\label{tab:reasoning_performance_final}
\small
\renewcommand{\arraystretch}{0.7}
\setlength{\tabcolsep}{4pt}
\begin{tabular*}{\textwidth}{@{\extracolsep{\fill}}llcccccc}
\toprule
\textbf{Rank ($r$)} & \textbf{Method} & \textbf{AIME 24} & \textbf{AIME 25} & \textbf{MATH-500} & \textbf{GPQA Dia.} & \textbf{LCB v4} & \textbf{Avg.} \\
&  & ($n=64$) & ($n=64$) & ($n=4$) & ($n=8$) & ($n=16$) & \\
\midrule
N/A & Full Fine-Tuning & 37.03 & 28.69 & 86.50 & 39.35 & 28.71 & 44.06 \\
\midrule
\multirow{3}{*}{64} 
& LoRA           & 23.07 & 19.73 & 83.00 & 34.63 & 14.85 & 35.06 \\
& LoRA ($10\times$ LR)   & \underline{28.29} & \textbf{25.83} & \textbf{86.00} & \textbf{37.82} & \textbf{19.80} & \textbf{39.55} \\
& {LoRA-$\alpha$} & \textbf{28.69} & \textbf{25.83} & \underline{85.80} & \underline{36.95} & \textbf{19.80} & \underline{39.41} \\
\midrule
\multirow{3}{*}{256}
& LoRA           & 26.51 & 23.69 & \underline{84.60} & 35.54 & 20.79 & 38.23 \\
& LoRA ($10\times$ LR)   & \underline{33.69} & \underline{26.71} & 83.00 & \underline{38.30} & \underline{22.77} & \underline{40.89} \\
& {LoRA-$\alpha$} & \textbf{34.21} & \textbf{28.95} & \textbf{85.80} & \textbf{38.50} & \textbf{26.73} & \textbf{42.84} \\
\bottomrule
\end{tabular*}
\end{table*}

\begin{table*}[t]
\renewcommand{\arraystretch}{0.7}
\setlength{\tabcolsep}{4pt}
\centering
\caption{Comparison of Pass@1 accuracy on reasoning benchmarks after GRPO learning. We train DeepSeek-R1-Distill-Qwen with a base learning rate of $1 \times 10^{-6}$ and a fixed LoRA rank of $64$.}
\label{tab:rl_performance}
\small
\begin{tabular}{llccccccc}
\toprule
\textbf{Model} & \textbf{Method} & \textbf{AIME 24} & \textbf{AIME 25} & \textbf{MATH-500} & \textbf{AMC 23} & \textbf{MINERVA} & \textbf{HMMT} & \textbf{Avg.} \\
& & ($n=64$) & ($n=64$) & ($n=4$) & ($n=32$) & ($n=8$) & ($n=32$) & \\
\midrule 
\multirow{5}{*}{1.5B} & Base model & 21.45 & 13.33 & 31.25 & 67.50 & 13.33 & \textbf{32.35} & 29.87 \\
& Full Fine-Tuning & \underline{29.79} & \underline{22.50} & \textbf{82.00} & 62.50 & \textbf{31.25} & 6.67 & \underline{39.12} \\
& LoRA & 22.29 & 6.67 & 32.35 & \textbf{72.50} & 27.57 & \underline{13.33} & 29.12 \\
& LoRA ($10\times$ LR) & 22.29 & 10.00 & 0.00 & 65.00 & 29.04 & \underline{13.33} & 23.27 \\
& {LoRA-$\alpha$} & \textbf{32.34} & \textbf{22.81} & \underline{79.80} & \underline{70.00} & \underline{29.41} & \underline{13.33} & \textbf{41.28} \\
\midrule
\multirow{5}{*}{7B} & Base model & 51.82 & \textbf{37.50} & 86.60 & \underline{90.00} & \underline{40.80} & \textbf{16.67} & \underline{53.90} \\
& Full Fine-Tuning & 38.95 & 16.67 & 42.65 & \textbf{92.50} & \underline{40.80} & \textbf{16.67} & 41.37 \\
& LoRA  & 51.82 & \underline{37.34} & 87.20 & \underline{90.00} & \textbf{41.54} & \underline{13.33} & 53.54 \\
& LoRA ($10\times$ LR) & \underline{53.33} & 36.51 & \underline{87.60} & 87.50 & 40.07 & \underline{13.33} & 53.06 \\
& {LoRA-$\alpha$} & \textbf{55.52} & 35.62 & \textbf{88.60} & \underline{90.00} & \textbf{41.54} & \textbf{16.67} & \textbf{54.66} \\
\bottomrule
\end{tabular}
\end{table*}

\subsection{Performance on Reasoning-based Supervised Fine-Tuning}

We evaluate LoRA-$\alpha$ in a long-context reasoning setting. Following~\cite{openr1}, we post-tune Qwen 2.5-Math-7B~\cite{qwenmath} on the Mixture-of-Thoughts dataset including 350k reasoning traces across math, programming, and science with a 32k token context. 
Training uses a batch size of 128, a learning rate of $4\times10^{-5}$, and LoRA ranks $r \in \{64, 256\}$ for one epoch, taking $\sim$2 days per run. 
We compare FFT, standard LoRA, LoRA with $10\times$ learning rate, and LoRA-$\alpha$ on AIME 2024/2025, MATH-500, GPQA Diamond, and LiveCodeBench v4 using Pass@1.

\textbf{Results.} As shown in Table~\ref{tab:reasoning_performance_final}, LoRA-$\alpha$ performs comparably to the $10\times$ learning rate variant at smaller ranks, but achieves stronger gains as the rank increases. In contrast, learning-rate scaling becomes unstable at higher ranks. Notably, due to the intrinsic difficulty of this task, LoRA-based methods do not surpass FFT, reflecting the substantial learning capacity required. Nevertheless, LoRA-$\alpha$ consistently emerges as the closest approximation to FFT.

\subsection{Performance on Reinforcement Learning Paradigm}

We evaluate LoRA-$\alpha$ under the reinforcement learning using {Group Relative Policy Optimization (GRPO)}~\cite{guo2025deepseek}. Following~\cite{yin2025evaluating}, 
experiments are conducted on {DeepSeek-R1-Distill-Qwen (1.5B and 7B)} models, trained on the {DAPO-Math-17k} dataset~\cite{yu2025dapo} with a {LoRA rank of $64$}, {learning rate of $1\times10^{-6}$} and {batch size of 128}. We use a {maximum completion length of 16k tokens} with {$G=8$ rollouts per prompt}, taking $\sim$2 days per run. 
Evaluation is conducted on a suite of challenging mathematical benchmarks, using Pass@1 as the primary metric.

\textbf{Results.}
As shown in Table~\ref{tab:rl_performance}, LoRA-$\alpha$ achieves the best overall performance across both model scales. In contrast, increasing the learning rate fails to yield reliable gains and often degrades performance, highlighting the instability induced by high-variance policy gradients. 
FFT also exhibits inconsistent behavior, with noticeable degradation on several benchmarks, suggesting that updating all parameters can disrupt pretrained reasoning structures. 
Combined with the results in Table~\ref{tab:reasoning_performance_final}, these findings echo the observations of~\cite{schulman2025lora} and suggest a broader implication: compared to SFT, reinforcement learning typically requires lower  learning capacity while being more susceptible to optimization instability. 
In this regime, LoRA is better suited than FFT, and properly scaling $\alpha$ preserves training stability and improves learning effectiveness.

\section{Conclusion}
This paper investigates the scaling mechanisms of LoRA, combining extensive empirical evidence with a novel Signal-Drift framework. We show that LoRA's stability stems from {spectral suppression} of the task Hessian, which creates a significant optimization gap. 
Crucially, we demonstrate that scaling $\alpha$ is more effective than increasing the learning rate, as it amplifies task-aligned signal without introducing additional drift. 
Building on this insight, we establish a sublinear scaling law and propose {LoRA-$\alpha$}, which restores $\alpha$ to its principled regime and enables direct reuse of FFT learning rates to substantially improve LoRA performance.
Overall, our work provides both theoretical insight and practical guidance for unlocking LoRA’s  potential through a simple yet effective scaling mechanism.

Looking ahead, developing novel adapter architectures that structurally mitigate this drift without compromising task-signal integrity remains an exciting frontier for future research.

\begingroup
\small
\bibliography{reference}
\bibliographystyle{icml_bib_style}

\endgroup

%%%%%%%%%%%%%%%%%%%%%%%%%%%%%%%%%%%%%%%%%%%%%%%%%%%%%%%%%%%%

%%%%%%%%%%%%%%%%%%%%%%%%%%%%%%%%%%%%%%%%%%%%%%%%%%%%%%%%%%%%

\newpage
\appendix

\section{Limitations}\label{sec:Limitations}
Our analysis of the asymmetric roles of the scaling factor $\alpha$ and the learning rate $\eta$ is grounded in the behavior of adaptive optimizers. 
While the Signal–Drift decomposition itself is optimizer-agnostic, the observed asymmetry, where $\alpha$ amplifies task-aligned signal while $\eta$ introduces additional drift, may not directly extend to vanilla SGD. 
We adopt this perspective to reflect practical training regimes, where adaptive optimizers are ubiquitous in deep learning and large model training.

\section{Broader Impact}\label{sec:Broader Impact}
This work studies the scaling mechanism of LoRA and proposes a principled scaling strategy that improves training effectiveness. By enabling more effective fine-tuning and reduced hyperparameter search, our method can lower the computational cost associated with adapting large models, potentially making such techniques more accessible.
At the same time, improved training effectiveness  may facilitate the broader deployment of large language models, including in settings where misuse is possible. Nevertheless, our work does not introduce new capabilities beyond improving optimization, and thus does not fundamentally alter the risk profile of existing models.
Overall, we view this work as a step toward more efficient and principled model adaptation, while emphasizing that downstream applications should be developed and deployed with appropriate safeguards.

\section{Related Work}

LoRA~\citep{Hu2021LoRALA} leverages low-rank decomposition to efficiently adapt large models, and has become a key technique for scalable deployment in modern large-scale systems, significantly reducing training~\cite{peft,schulman2025lora,wenbinwang2026loraoptimalrank} and serving costs~\cite{kwon2023efficient}. Comprehensive overviews of LoRA and its extensions can be found in~\citep{yang2024low,mao2025survey,he2026unified}. 
To improve the expressiveness of low-rank updates, prior works explore architectural enhancements such as stacking strategies~\citep{lialin2023relora,xia2024chain}, higher-order parameterizations~\citep{Edalati2022KronAPE}, and customized adapter designs~\citep{jiang2024mora,liu2024dora,zhong2024neat}. 
More relevant to our work, research focuses on improving optimization and training efficiency. Existing approaches investigate scaling strategies~\citep{kalajdzievski2023rank,biderman2024lora}, learning rate design~\citep{lora+}, initialization schemes~\citep{PiSSA,lora-ga,zhang2025primacy}, optimizer choices~\citep{li2024crucial}, and batch size configurations~\cite{lee2026beware}. 

\textbf{Optimization Dynamics of LoRA.}
The optimization behavior of LoRA remains poorly understood due to its bilinear parameterization and non-convex landscape~\citep{li2024crucial,xu2025understanding}. 
Existing theoretical analyses are often restricted to simplified regimes, such as lazy training or infinite-width limits~\citep{malladi2023kernel,hayou2024impact}. 
\citet{kalajdzievski2023rank,RoRA} study stability conditions, showing that improper scaling can lead to gradient collapse or inefficient training. 
~\citep{lora+,zhu2024asymmetry} further highlight structural asymmetries between the two low-rank factors, motivating asymmetric optimization strategies such as distinct learning rates.  ~\cite{mu2025convergence} studies a non-asymptotic convergence analysis of the original LoRA gradient descent algorithm.
While these efforts provide valuable insights, they do not fully explain how scaling choices affect optimization efficiency in practical settings.

\textbf{The $\eta$--$\alpha$ Relationship.}
In practice, LoRA is typically trained with a fixed scaling factor (often $\alpha \propto r$) while tuning the learning rate $\eta$~\cite{Hu2021LoRALA}. 
Empirically, LoRA requires significantly larger learning rates, often exceeding $10\times$ those of full fine-tuning, to achieve competitive performance~\citep{biderman2024lora,schulman2025lora}. 
Recent theoretical results reveal a coupling between $\eta$ and $\alpha$~\citep{zhang2025primacy,schulman2025lora}, showing that \textit{increasing $\alpha$ can mimic the effect of simultaneously scaling  the learning rate and initialization under adaptive optimizers. }
We are inspired by these theoretical analyses, which we interpret as a hint that $\alpha$ and $\eta$ are not interchangeable, with $\alpha$ exerting a stronger influence on optimization than the learning rate.
More recently, \citet{chen2026learning} propose the Maximal-Update Adaptation framework to study rank-invariant learning rates, while \citet{lee2026learning} show that different LoRA variants favor distinct learning rate ranges despite achieving similar peak performance.
Nevertheless, these studies are largely confined to the regime $\alpha = r$. Our empirical results show that this constraint limits LoRA’s fitting capacity and leaves a significant optimization gap.

\paragraph{Discussion.} 
\textit{To the best of our knowledge, this is the first work to systematically characterize the role of the scaling factor $\alpha$ in LoRA optimization.}
While existing analyses emphasize learning rate or initialization, we demonstrate that $\alpha$ plays a fundamentally different role. 
From a Hessian perspective, we explain why LoRA requires aggressive hyperparameters, and reveal a key asymmetry: $\alpha$ primarily amplifies task-aligned curvature, whereas $\eta$ amplifies the magnitude of bilinear-induced drift. 
This decomposition suggests an optimal regime that increases $\alpha$ while keeping $\eta$ conservative, enabling stronger signal amplification without destabilizing optimization.

\section{Details for Empirical Study}
\label{appendix:more_fitting}

In this section, we provide comprehensive details of the experimental setup and extended optimization results across different model scales and datasets.

\subsection{Experimental Setup and Search Strategy}

\textbf{Hyperparameter Search Strategy.} 
To characterize the optimization landscape of LoRA, we tailored our hyperparameter search strategies based on the dataset and our evolving empirical understanding:
\begin{itemize}[leftmargin=*]
    \item \textbf{Tulu 3:} Given the lack of established heuristics for the optimal scaling regime at the onset of our study, we adopted a two-stage approach for the Tulu 3 dataset. We first defined a broad, coarse-grained grid across the learning rate $\eta$ and the scaling factor $\alpha$. Upon identifying the general region of convergence, we conducted a finer-grained localized search around the minimum validation loss to accurately pinpoint the optimal configurations $\eta^*$ and $\alpha^*$.
    \item \textbf{OpenThoughts:} For the OpenThoughts dataset, rather than repeating an exhaustive grid search, we leveraged the empirical law derived from the Tulu 3 experiments, where the base scaling coefficient was identified near $C=256$. We anchored our search space around this empirical baseline, systematically sweeping $\alpha$ within a $10\times$ range (both scaling up and down) to validate the transferability of the sublinear scaling law.
\end{itemize}

\textbf{OpenThoughts Dataset Details.} 
To evaluate whether our findings generalize to reasoning-heavy tasks, which typically exhibit different optimization dynamics than general instruction tuning, we utilized the OpenThoughts dataset \cite{guha2025openthoughts}. We sampled a 10k training subset and a 1k validation set. Because reasoning traces inherently require substantially more context, we extended the maximum sequence length to $T=8192$ (compared to $T=1024$ for Tulu 3). Note that FFT experiments were omitted for OpenThoughts, as the primary objective here is to validate the relative scaling dynamics between $\alpha$ and $\eta$ established in the main text.

\begin{figure}[h]
    \centering
\includegraphics[width=1.0\linewidth]{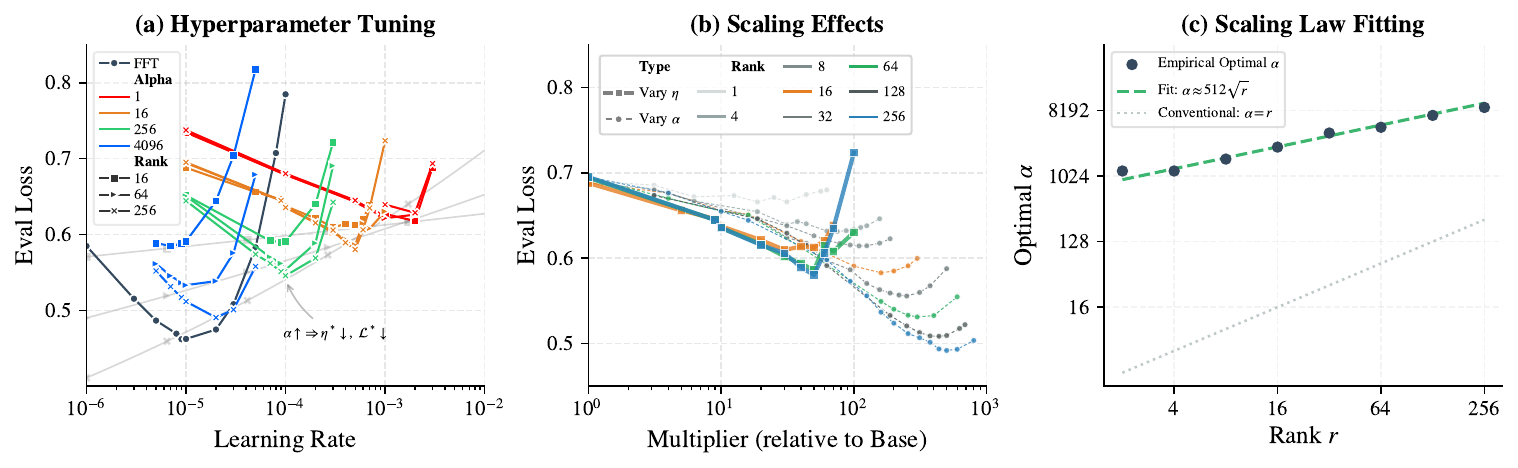}
\caption{Hyperparameter analysis of Llama 3-8B on the Tulu 3 dataset. 
(a) Evaluation loss as a function of the learning rate $\eta$ across different ranks $r$ and scaling factors $\alpha$. Gray lines denote linear fits of the minimum loss for each $(r, \alpha)$. Increasing $\alpha$ lowers the optimal loss and shifts $\eta^*$ downward.
(b) Scaling paths defined in Eq.~\eqref{eq:scaling_path} relative to the baseline ($\eta_{\text{FFT}} = 1\times10^{-5}$, $\alpha_0 = 16$), where $\eta_0$ denotes the optimal FFT learning rate. Varying $\alpha$ reaches lower-loss regimes that are inaccessible via $\eta$-tuning alone. 
(c) Optimal scaling factor $\alpha^*$ as a function of rank $r$, evaluated at $\eta_0$. The observed sublinear trend and large magnitude of $\alpha$ challenge conventional scaling heuristics.}
    \label{fig:llama3_tulu3_appendix}
\end{figure}

\begin{figure}[h]
    \centering
    \includegraphics[width=1.0\linewidth]{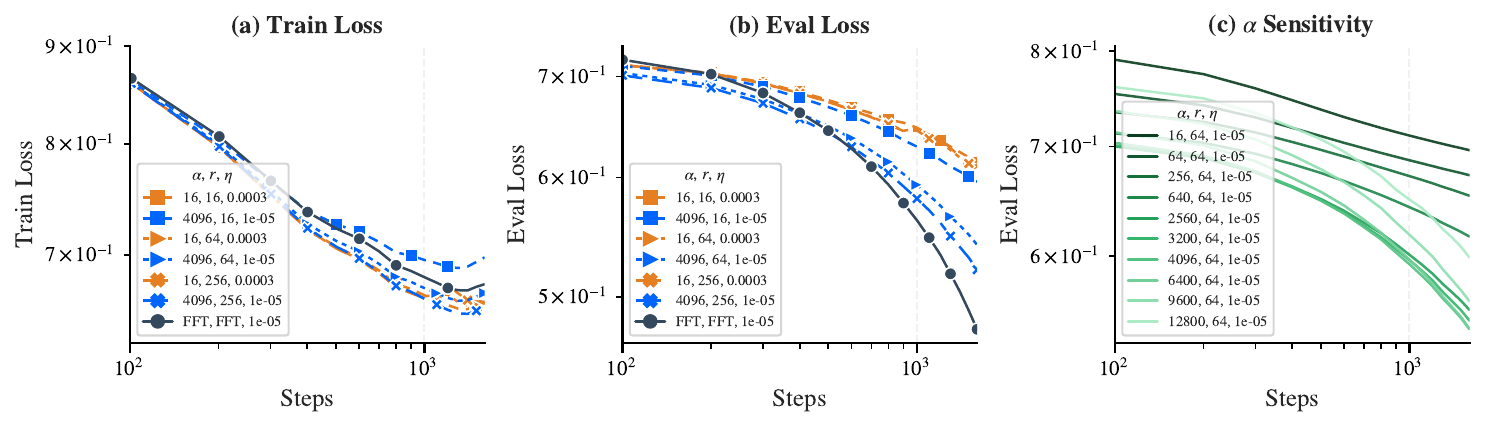}
    \caption{Optimization dynamics and $\alpha$ sensitivity for Llama 3-8B on the Tulu 3 dataset over training steps. (a) Training loss and (b) Evaluation loss curves for various combinations of rank $r$, scaling factor $\alpha$, and learning rate $\eta$, compared alongside FFT. Properly scaled configurations (\textit{e.g.}, larger $\alpha$ with standard $\eta$) tightly approximate the FFT trajectory. (c) Sensitivity analysis of the scaling factor $\alpha$, illustrating how increasing $\alpha$ smoothly accelerates convergence.}
    \label{fig:llama3-8b_loss}
\end{figure}

\subsection{Scaling Behavior of Llama 3-8B on Tulu 3}

Fig.~\ref{fig:llama3_tulu3_appendix} illustrates the hyperparameter analysis for the Llama 3-8B model on the Tulu 3 dataset. The observations strictly corroborate our findings from the 1B model. Notably, $\alpha$-scaling continues to demonstrate superior optimization capacity over $\eta$-scaling. The optimal scaling factor $\alpha^*$ fits a sublinear relationship $\alpha = 512\sqrt{r}$, confirming that larger models also operate optimally at an $\alpha$ magnitude substantially larger than conventional rank-tied heuristics.

Furthermore, Fig.~\ref{fig:llama3-8b_loss} provides a granular view of the optimization dynamics over training steps. The loss trajectories demonstrate that properly scaled configurations, \textit{i.e.}, operating with a large $\alpha$ and a standard small $\eta$, tightly approximate the descent trajectory of FFT. 

\subsection{Validation on OpenThoughts Reasoning Tasks}

Fig.~\ref{fig:1B-openthoughts} and \ref{fig:8B-openthoughts} present the optimization behaviors on the OpenThoughts dataset for Llama 3-1B and 8B, respectively. Despite the increased sequence length and task complexity, the fundamental asymmetry between $\alpha$ and $\eta$ remains persistent. Both models consistently avoid early saturation and reach deeper loss minima when $\alpha$ is scaled. Interestingly, the fitted optimal scaling laws yield large coefficients $C=1024$, indicating that complex reasoning tasks may demand an even more aggressive restoration of the task signal curvature.

\begin{figure}[h]
    \centering
\includegraphics[width=0.75\linewidth]{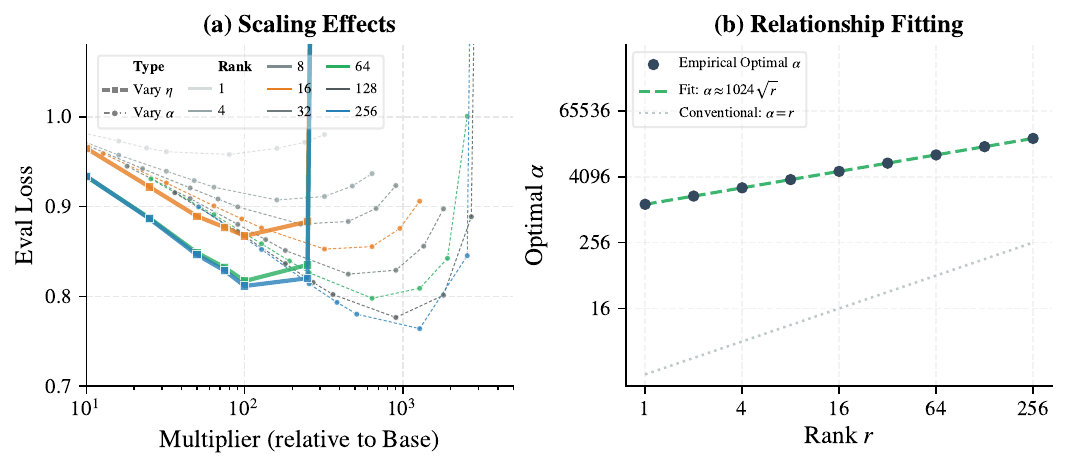}
    \caption{Hyperparameter analysis of Llama 3-1B on the OpenThoughts dataset with  $\eta_{\text{FFT}}=1\times 10^{-5}$ and $\alpha_0=16$. (a) Scaling paths comparing $\alpha$-scaling and $\eta$-scaling. Consistent with the main findings, increasing $\alpha$ enables the model to reach deeper loss minima that are inaccessible via $\eta$-tuning alone. (b) Optimal scaling factor $\alpha^*$ as a function of rank $r$. The fitted sublinear relationship confirms the scaling law established in the main text, demonstrating its robustness across different datasets.}
    \label{fig:1B-openthoughts}
\end{figure}

\begin{figure}[h]
    \centering
    \includegraphics[width=0.75\linewidth]{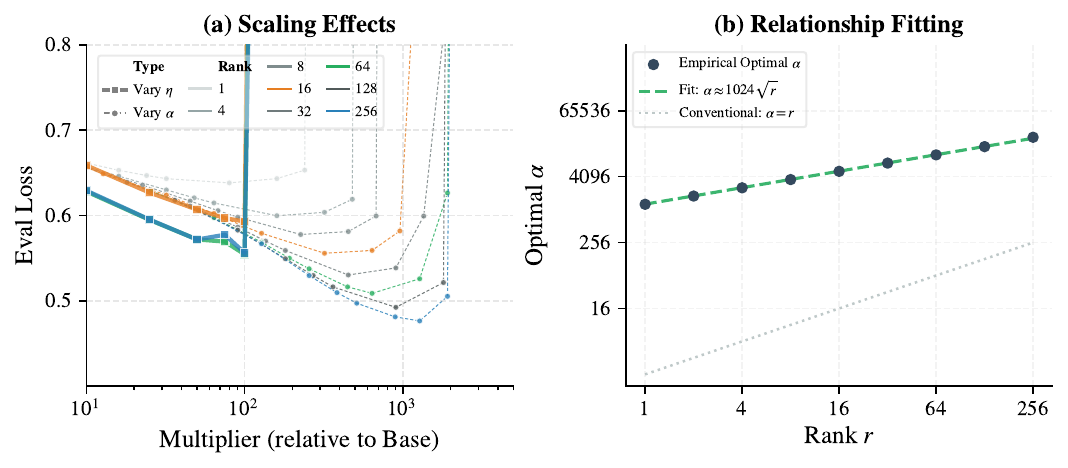}
    \caption{Hyperparameter analysis of Llama 3-8B on the OpenThoughts dataset with  $\eta_{\text{FFT}}=1\times 10^{-5}$ and $\alpha_0=16$. (a) Scaling paths comparing $\alpha$-scaling and $\eta$-scaling. Consistent with the main findings, increasing $\alpha$ enables the model to reach deeper loss minima that are inaccessible via $\eta$-tuning alone. (b) Optimal scaling factor $\alpha^*$ as a function of rank $r$. The fitted sublinear relationship confirms the scaling law established in the main text, demonstrating its robustness across different datasets.}
    \label{fig:8B-openthoughts}
\end{figure}

\section{Empirical Validation of Signal-Drift Dynamics}
\label{appendix:toy_example_details}

To empirically validate the Signal-Drift dynamics and visualize the intractable Hessian matrices without the memory bottleneck, we construct a controlled environment using a multi-layer perceptron (MLP). This setup allows us to precisely track the optimization trajectories and perform exact spectral analysis on the optimization landscape.

\paragraph{Setups.} 
We employ a 5-layer MLP utilizing the SiLU activation function, trained to fit synthetic Gaussian data. We generate a target projection matrix $W_{\text{true}}$ constructed via Singular Value Decomposition (SVD), where the singular values follow a strict decay rate of $\sigma_i = 1/i$. The training set consists of $n=100$ samples drawn from a standard normal distribution, with targets generated as $Y = XW_{\text{true}}^\top$. LoRA adapters are applied to all linear layers with a fixed rank of $r=16$. Pre-trained weights $W_0$ and adapter matrices $A$ are initialized from $\mathcal{N}(0, 1/d)$, while $B$ is initialized to zero. For general optimization tracking, we use a hidden dimension of $d=256$. For exact Hessian spectral analysis, which scales quadratically with parameter count, we reduce the dimension to $d=32$. 

\paragraph{Optimization Dynamics ($\eta$-scaling vs. $\alpha$-scaling).} 
The model is optimized using Adam for 100 steps with a MSE loss. The base learning rate and base scaling factor are set to $\eta_0 = 10^{-4}$ and $\alpha_0 = 1$, respectively. To compare their behavioral differences, we conduct experiments across two distinct regimes:
\begin{itemize}[leftmargin=*, topsep=0em, itemsep=0em]
    \item \textbf{$\eta$-scaling:} We fix $\alpha = \alpha_0$ and strictly scale the learning rate by factors of $\{16, 32, 64, 128\} \times \eta_0$.
    \item \textbf{$\alpha$-scaling:} We anchor the learning rate at $\eta = \eta_0$ and scale $\alpha$ by factors of $\{32, 64, 128, 256\} \times \alpha_0$.
\end{itemize}
These two regimes yield fundamentally different optimization dynamics. While configurations $\eta = 64\eta_0$ and $\alpha = 128\alpha_0$ initially exhibit comparable convergence trends, the $\eta$-scaled trajectory already manifests noticeable oscillations. Pushing the scaling further exposes the structural bottleneck: $\eta = 128\eta_0$ causes the training to catastrophically diverge, whereas $\alpha = 256\alpha_0$ remains strictly stable and continues to unlock deeper fitting (as evidenced in Fig.~\ref{fig:sdr_convergence}).

\begin{figure}[h]
    \centering
    \includegraphics[width=1.0\linewidth]{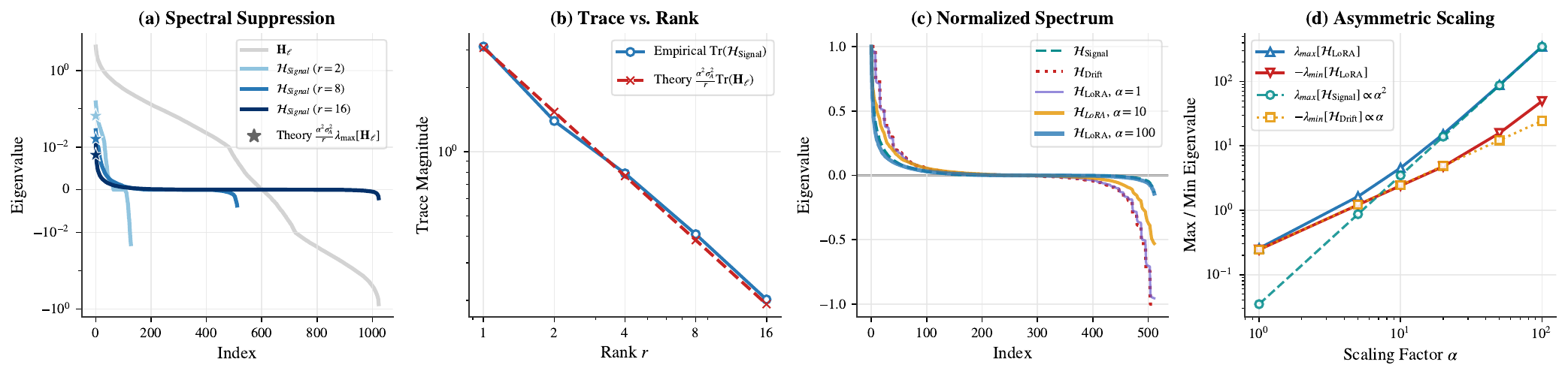}
    \caption{Spectral analysis of Hessian matrices.
    (a) Comparison between the full-parameter Hessian $\mathbf{H}_{\ell}$ and LoRA's signal Hessian $\mathcal{H}_{\text{Signal}}$, illustrating the compression of total curvature energy.
    (b) The predicted sum of eigenvalues from our Signal-Drift framework closely matches the empirical values across various ranks, validating our theoretical derivation.
    (c) The exact LoRA Hessian $\mathcal{H}_{\text{LoRA}}$ is inherently indefinite due to the structural drift $\mathcal{H}_{\text{Drift}}$. Increasing $\alpha$ triggers spectral purification, where the spectrum converges toward the positive semi-definite $\mathcal{H}_{\text{Signal}}$. 
    (d) Empirical verification of asymmetric scaling: the signal curvature magnitude grows as $\Theta(\alpha^2)$ while drift noise grows only as $\Theta(\alpha)$. Consequently, as $\alpha$ increases, $\lambda_{\max}[\mathcal{H}_{\text{LoRA}}]$ perfectly aligns with $\lambda_{\max}[\mathcal{H}_{\text{Signal}}]$.}
  \label{fig:spectral_properties}
\end{figure}

\paragraph{Spectral Compression under Low-Rank Parameterization.} To understand why LoRA scales differently from FFT, we explicitly compute the full-parameter Hessian $\mathbf{H}_{\ell}$, the LoRA Hessian $\mathcal{H}_{\text{LoRA}}$, and its components $\mathcal{H}_{\text{Signal}}$ and $\mathcal{H}_{\text{Drift}}$ via exact automatic differentiation. As shown in Fig.~\ref{fig:spectral_properties}(a), the spectrum of $\mathcal{H}_{\text{Signal}}$ is a compressed version of $\mathbf{H}_{\ell}$. The dominant eigenvalues are significantly reduced, and the trace explicitly follows $\operatorname{Tr}(\mathcal{H}_{\text{Signal}}) \approx \frac{\alpha^2\sigma_A^2}{r} \operatorname{Tr}(\mathbf{H}_{\ell})$. Notably, Fig.~\ref{fig:spectral_properties}(a) illustrates that the maximum eigenvalue $\lambda_{\max}$ empirically aligns with the predicted order of magnitude, although it is less strictly constrained than the trace. Since the trace equals the sum of all eigenvalues ($\operatorname{Tr}(\mathbf{H}) = \sum_i \lambda_i$), this trace-level compression inherently necessitates a suppression of the dominant spectral components. This confirms that low-rank parameterization severely restricts curvature energy, providing a quantitative explanation for why a larger $\alpha$ is fundamentally required to compensate for rank-induced suppression.

\paragraph{Indefiniteness and Spectral Purification.}
Unlike the full Hessian, $\mathcal{H}_{\text{LoRA}}$ is inherently indefinite due to the presence of the structural drift $\mathcal{H}_{\text{Drift}}$. As shown in Fig.~\ref{fig:spectral_properties}(c), this drift component introduces negative eigenvalues, which destabilize optimization. However, increasing $\alpha$ progressively amplifies the signal term relative to the drift, leading to a \emph{spectral purification} effect: the overall indefinite spectrum becomes increasingly aligned with the positive semi-definite $\mathcal{H}_{\text{Signal}}$.

\paragraph{Asymmetric Scaling of Signal and Drift.}
Fig.~\ref{fig:spectral_properties}(d) empirically verifies the asymmetric scaling behavior. We observe a strict divergence in growth rates:
\begin{align}
\lambda_{\max}[\mathcal{H}_{\text{Signal}}] &= \Theta(\alpha^2), \\
|\lambda_{\min}[\mathcal{H}_{\text{Drift}}]| &= \Theta(\alpha).
\end{align}
This mismatch guarantees that increasing $\alpha$ disproportionately strengthens the task-aligned curvature while only linearly amplifying the destabilizing drift. As a result, the dominant eigenvalue of $\mathcal{H}_{\text{LoRA}}$ increasingly aligns with that of $\mathcal{H}_{\text{Signal}}$.

Together, these controlled experiments provide a validated spectral explanation for the empirical superiority of $\alpha$-scaling: (i) low-rank parameterization suppresses curvature magnitude; (ii) $\alpha$ restores signal strength in a controlled manner; and (iii) the asymmetric scaling between signal and drift naturally dictates that increasing $\alpha$, rather than the learning rate, is the principled path to stable and accelerated optimization.

\section{Comparison of Scaling Factor  Magnitudes}\label{appendix:Comparison_Scaling_Factor}
\begin{table}[ht]
\centering
\caption{Comparison of the standard heuristic ($\alpha=r$) with our proposed empirical value ($256\sqrt{r}$) and  analytic value ($\frac{1}{\sigma_A}\sqrt{r} = \sqrt{3d_{in}r}$ with $d_{in}=4096$).}
\label{tab:magnitude_comparison}
\small
\begin{tabular}{@{}ccccc@{}}
\toprule
\textbf{Rank ($r$)} & \textbf{Heuristic ($\alpha=r$)} & \textbf{Empirical ($256\sqrt{r}$)} & \textbf{Theoretical ($\frac{1}{\sigma_A}\sqrt{r}$)} & \textbf{Magnitude Increase} \\ \midrule
1                   & 1                               & 256                               &  111                            & 256.0 $\times$              \\
4                   & 4                               & 512                               &  222                            & 128.0 $\times$              \\
8                   & 8                               & 724                               &  314                            & 90.5 $\times$               \\
16                  & 16                              & 1024                              &  443                            & 64.0 $\times$               \\
32                  & 32                              & 1448                              &  627                            & 45.3 $\times$               \\
64                  & 64                              & 2048                              &  887                            & 32.0 $\times$               \\
128                 & 128                             & 2896                              &  1254                           & 22.6 $\times$               \\
256                 & 256                             & 4096                              & 1774                           & 16.0 $\times$               \\ \bottomrule
\end{tabular}
\end{table}

\paragraph{Quantitative Comparison of Scaling Magnitudes.} To further illustrate the optimization gap, we provide a quantitative comparison between our proposed scaling law and prevailing heuristics in Table~\ref{tab:magnitude_comparison}. Conventional practices typically set $\alpha = r$, which implicitly confines the scaling factor to a severely under-scaled regime as the rank increases.
According to our Signal-Drift framework, restoring the curvature energy (Hessian trace) to a level comparable to FFT requires $\alpha$ to scale with $\frac{1}{\sigma_A}\sqrt{r}$. For a typical model dimension $d_{in}=4096$ and $\sigma_A^2 = 1/(3d_{in})$, the theoretical coefficient $\frac{1}{\sigma_A}$ is approximately $111$. Our empirical default $C=256$ is an even more optimized choice that ensures deep fitting. As shown in the table, for a standard rank $r=8$, LoRA-$\alpha$ provides a scaling magnitude that is 90.5$\times$ larger than the common heuristic. Even at $r=256$, the gap remains as high as 16$\times$. This comparison clarifies why standard FFT learning rates have historically been inadequate for driving LoRA to its full potential under legacy scaling rules.

\paragraph{Note on Initialization Variance.} The parameter $\sigma_A^2 = 1/(3d_{in})$ is chosen to align with the default initialization protocol in standard deep learning libraries (\textit{e.g.}, PyTorch and PEFT). Specifically, LoRA adapters are typically initialized using a Kaiming uniform distribution $\mathcal{U}(-\sqrt{1/d_{in}}, \sqrt{1/d_{in}})$. For a uniform distribution $\mathcal{U}(-k, k)$, the variance is given by $k^2/3$. By substituting $k^2 = 1/d_{in}$, we obtain $\sigma_A^2 = 1/(3d_{in})$. This ensures that the magnitudes predicted by our theoretical framework are consistent in order of magnitude with those used in practical implementations.

\section{Theoretical Proofs}\label{appendix:proof}

\textbf{Notations.} Consider a pre-trained weight $W_0 \in \mathbb{R}^{d_{out} \times d_{in}}$ and trainable adapters ${\theta} = \{A, B\}$, where $B \in \mathbb{R}^{d_{out} \times r}$ and $A \in \mathbb{R}^{r \times d_{in}}$. The LoRA mapping is defined as $W(\bm{\theta}) = W_0 + \frac{\alpha}{r}BA$. Let $\bm{w}(\bm{\theta}) = \operatorname{vec}(W(\bm{\theta}))\in \mathbb{R}^{D}$ denote the vectorized weight, and $\bm{\theta} =[\operatorname{vec}(A); \operatorname{vec}(B)] \in \mathbb{R}^p$ the concatenated parameters, where $D = d_{out}d_{in}$ and $p = r(d_{in} + d_{out})$.
For the objective $\ell(\bm{w})$, we define the task gradient $\bm{g} = \nabla_{\bm{w}} \ell \in \mathbb{R}^D$ and the task Hessian $\mathbf{H}_\ell = \nabla^2_{\bm{w}} \ell \in \mathbb{R}^{D \times D}$, with $g_k$ and $w_k$ representing their $k$-th scalar entries. The mapping geometry is governed by the Jacobian $J(\bm{\theta}) = \nabla_{\bm{\theta}} \bm{w} \in \mathbb{R}^{D \times p}$ and the structural Hessian $\nabla^2_{\bm{\theta}} \bm{w} \in \mathbb{R}^{D \times p \times p}$, interpreted as a bilinear operator: for any $\bm{v} \in \mathbb{R}^p$, $\nabla^2_{\bm{\theta}} \bm{w}[\bm{v}, \bm{v}] \in \mathbb{R}^D$ has entries $\bm{v}^\top (\nabla^2_{\bm{\theta}} w_k)\bm{v}$.

\subsection{Proof for Proposition~\ref{prop:sd_decomp}}\label{appendix:sd_decomp}
\begin{proposition*}[Signal-Drift Decomposition]
Let $\Delta \bm{w}_{\text{LoRA}} = \bm{w}(\bm{\theta} + \Delta\bm{\theta}) - \bm{w}(\bm{\theta})$ denote the effective step in the weight space, and $\mathcal{H}_{\text{LoRA}} = \nabla^2_{\bm{\theta}} \ell$ denote the parameter-space Hessian. Under the LoRA parameterization, both admit an exact decomposition into a task-aligned signal and a structural drift:
\begin{align}
\Delta \bm{w}_{\text{LoRA}}
= \underbrace{J(\bm{\theta})\Delta \bm{\theta}}_{\Delta \bm{w}_{\text{Signal}}}
+ \underbrace{\tfrac{1}{2}\nabla^2_{\bm{\theta}} \bm{w}[\Delta \bm{\theta}, \Delta \bm{\theta}]}_{\Delta \bm{w}_{\text{Drift}}}, \
\mathcal{H}_{\text{LoRA}} 
= \underbrace{J(\bm{\theta})^\top \mathbf{H}_\ell J(\bm{\theta})}_{\mathcal{H}_{\text{Signal}}}
+ \underbrace{\Sigma_{k=1}^{D} g_k \nabla^2_{\bm{\theta}} w_k}_{\mathcal{H}_{\text{Drift}}}.
\end{align}
\end{proposition*}

\begin{proof}

\textbf{Part I: Decomposition of the Weight Update $\Delta \bm{w}$.} \\
Recall that the general multivariate Taylor series expansion of the mapping $\bm{w}$ around $\bm{\theta}$ for a parameter update $\Delta\bm{\theta}$ is given by:
\begin{equation}
    \bm{w}(\bm{\theta} + \Delta\bm{\theta}) = \bm{w}(\bm{\theta}) + \nabla_{\bm{\theta}} \bm{w}(\bm{\theta}) \Delta\bm{\theta} + \frac{1}{2} \nabla^2_{\bm{\theta}} \bm{w}[\Delta\bm{\theta}, \Delta\bm{\theta}] + \mathcal{O}(\|\Delta\bm{\theta}\|^3).
\end{equation}
Under the LoRA framework, the mapping from the parameters $\bm{\theta} \coloneqq [\operatorname{vec}(A); \operatorname{vec}(B)]$ to the vectorized weights $\bm{w}$ is strictly bilinear. 
Because each element of $\bm{w}(\bm{\theta})$ contains products of at most two parameter elements, all third-order and higher-order derivatives of $\bm{w}$ with respect to $\bm{\theta}$ vanish identically. Consequently, the higher-order residual term $\mathcal{O}(\|\Delta\bm{\theta}\|^3)$ is exactly zero, and the expansion terminates exactly at the second-order term. Thus, we have the exact equality:
\begin{equation}
    \bm{w}(\bm{\theta} + \Delta\bm{\theta}) = \bm{w}(\bm{\theta}) + \nabla_{\bm{\theta}} \bm{w}(\bm{\theta}) \Delta\bm{\theta} + \frac{1}{2} \nabla^2_{\bm{\theta}} \bm{w}[\Delta\bm{\theta}, \Delta\bm{\theta}].
\end{equation}
By definition, the weight displacement is $\Delta \bm{w}_{\text{LoRA}} \coloneqq \bm{w}(\bm{\theta} + \Delta\bm{\theta}) - \bm{w}(\bm{\theta})$ and the Jacobian is $J(\bm{\theta}) \coloneqq \nabla_{\bm{\theta}} \bm{w}(\bm{\theta})$. Substituting these definitions directly yields the first decomposition:
\begin{equation}
    \Delta \bm{w}_{\text{LoRA}} = \underbrace{J(\bm{\theta})\Delta \bm{\theta}}_{\Delta \bm{w}_{\text{Signal}}} + \underbrace{\frac{1}{2}\nabla^2_{\bm{\theta}} \bm{w}[\Delta \bm{\theta}, \Delta \bm{\theta}]}_{\Delta \bm{w}_{\text{Drift}}}.
\end{equation}

\textbf{Part II: Decomposition of the Hessian $\mathcal{H}_{\text{LoRA}}$.} \\
First, we apply the multivariate chain rule to find the gradient of $\ell$ with respect to the $i$-th parameter $\theta_i$:
\begin{equation}
    \frac{\partial \ell}{\partial \theta_i} = \sum_{k=1}^D \frac{\partial \ell}{\partial w_k} \frac{\partial w_k}{\partial \theta_i} = \sum_{k=1}^D g_k \frac{\partial w_k}{\partial \theta_i},
\end{equation}
where $g_k = \frac{\partial \ell}{\partial w_k}$ is the $k$-th element of the task gradient $\bm{g}$.

Next, we differentiate this expression with respect to another parameter $\theta_j$ to obtain the $(i, j)$-th entry of the parameter Hessian:
\begin{equation}
    [\mathcal{H}_{\text{LoRA}}]_{i,j} = \frac{\partial^2 \ell}{\partial \theta_i \partial \theta_j} = \frac{\partial}{\partial \theta_j} \left( \sum_{k=1}^D g_k \frac{\partial w_k}{\partial \theta_i} \right).
\end{equation}
Applying the product rule, we obtain two separate terms:
\begin{equation}
    [\mathcal{H}_{\text{LoRA}}]_{i,j} = \sum_{k=1}^D \left( \frac{\partial g_k}{\partial \theta_j} \frac{\partial w_k}{\partial \theta_i} + g_k \frac{\partial^2 w_k}{\partial \theta_i \partial \theta_j} \right).
\end{equation}
We now expand the term $\frac{\partial g_k}{\partial \theta_j}$ using the chain rule once more:
\begin{equation}
    \frac{\partial g_k}{\partial \theta_j} = \sum_{m=1}^D \frac{\partial g_k}{\partial w_m} \frac{\partial w_m}{\partial \theta_j} = \sum_{m=1}^D [\mathbf{H}_\ell]_{k,m} J_{m,j},
\end{equation}
where $[\mathbf{H}_\ell]_{k,m} = \frac{\partial^2 \ell}{\partial w_k \partial w_m}$ is the $(k, m)$-th entry of the task Hessian $\mathbf{H}_\ell$, and $J_{m,j} = \frac{\partial w_m}{\partial \theta_j}$ is the $(m, j)$-th entry of the Jacobian $J(\bm{\theta})$.
Substituting this expansion back into the expression for $[\mathcal{H}_{\text{LoRA}}]_{i,j}$:
\begin{equation}
    [\mathcal{H}_{\text{LoRA}}]_{i,j} = \sum_{k=1}^D \sum_{m=1}^D \underbrace{\frac{\partial w_k}{\partial \theta_i}}_{J_{k,i}} [\mathbf{H}_\ell]_{k,m} \underbrace{\frac{\partial w_m}{\partial \theta_j}}_{J_{m,j}} + \sum_{k=1}^D g_k \frac{\partial^2 w_k}{\partial \theta_i \partial \theta_j}.
\end{equation}
Observe that the double summation on the left corresponds exactly to the matrix multiplication $(J(\bm{\theta})^\top \mathbf{H}_\ell J(\bm{\theta}))_{i,j}$. The summation on the right corresponds exactly to the $(i, j)$-th entry of the matrix $\sum_{k=1}^D g_k \nabla^2_{\bm{\theta}} w_k$.

Converting this element-wise equality back into matrix notation, we arrive at the final decomposition:
\begin{equation}
    \mathcal{H}_{\text{LoRA}} = \underbrace{J(\bm{\theta})^\top \mathbf{H}_\ell J(\bm{\theta})}_{\mathcal{H}_{\text{Signal}}} + \underbrace{\sum_{k=1}^{D} g_k \nabla^2_{\bm{\theta}} w_k}_{\mathcal{H}_{\text{Drift}}}.
\end{equation}
This establishes both partitions and completes the proof.
\end{proof}

By the definition of the LoRA forward pass, the weight matrix is $W(\theta) = W + \frac{\alpha}{r}BA$. For a parameter perturbation $\Delta \theta = \{\Delta A, \Delta B\}$, the perturbed weight matrix is:
\begin{align}
    W(\theta + \Delta \theta) &= W + \frac{\alpha}{r} (B + \Delta B)(A + \Delta A) \\
    &= W + \frac{\alpha}{r}BA + \frac{\alpha}{r}(B\Delta A + \Delta BA) + \frac{\alpha}{r}\Delta B\Delta A.
\end{align}
The first-order variation with respect to the parameters corresponds strictly to the linear terms in $\Delta A$ and $\Delta B$. Vectorizing this linear component directly yields the exact algebraic form of the task signal:
\begin{equation}
    \Delta \bm{w}_{\text{Signal}} = J(\bm{\theta})\Delta \bm{\theta} = \operatorname{vec}\left( \frac{\alpha}{r} (B \Delta A + \Delta B A) \right).
\end{equation}

 Vectorizing this bilinear cross-term gives the structural drift:
\begin{equation}
    \Delta \bm{w}_{\text{Drift}} = \frac{1}{2}\nabla^2_{\bm{\theta}} \bm{w}[\Delta \bm{\theta}, \Delta \bm{\theta}] = \operatorname{vec}\left( \frac{\alpha}{r} \Delta B \Delta A \right).
\end{equation}

\subsection{Proof for Proposition~\ref{prop:geometric_properties}}
\label{app:proof_remark}

\begin{proposition*}[Geometric Properties of Signal and Drift]
For any gradient-based update step $\Delta \bm{\theta} = -\eta \nabla_{\bm{\theta}} \ell$, the decomposed components in Proposition~\ref{prop:sd_decomp} satisfy: 

\textbf{(i) Constructive Signal:} The signal component  aligns with the  descent direction, $\langle \Delta \bm{w}_{\text{Signal}}, -\bm{g} \rangle \ge 0$, and preserves the local convexity of the loss landscape, \textit{i.e.}, $\mathcal{H}_{\text{Signal}}  \succeq 0$ given $\mathbf{H}_\ell \succeq 0$.

\textbf{(ii) Adversarial Drift:} The structural Hessian $\mathcal{H}_{\text{Drift}}$ is strictly indefinite, and its induced update $\Delta \bm{w}_{\text{Drift}}$ exerts an uncontrolled force with no guaranteed alignment with $-\bm{g}$.
\end{proposition*}

\begin{proof}
\textbf{Part I: Properties of the Signal Terms.}

\textit{1. Gradient alignment ($\langle \Delta \bm{w}_{\text{Signal}}, -\bm{g} \rangle \ge 0$):} \\
We evaluate the inner product between the signal component and the negative task gradient $-\bm{g}$ (the ideal steepest descent direction in the full-weight space). By the definition of the signal term and the chain rule ($\nabla_{\bm{\theta}} \ell = J(\bm{\theta})^\top \bm{g}$), this inner product elegantly reduces to the parameter space:
\begin{equation}
    \langle \Delta \bm{w}_{\text{Signal}}, -\bm{g} \rangle = \langle J(\bm{\theta}) \Delta \bm{\theta}, -\bm{g} \rangle = \langle \Delta \bm{\theta}, -J(\bm{\theta})^\top \bm{g} \rangle = \langle \Delta \bm{\theta}, -\nabla_{\bm{\theta}} \ell \rangle.
\end{equation}
This establishes that the alignment of the projected signal in the full-weight space is mathematically equivalent to the alignment of the parameter update with the parameter gradient. We analyze this property under three distinct optimization regimes with learning rate $\eta > 0$:

\begin{itemize}[leftmargin=*, topsep=0.5em, itemsep=0.5em]
    \item \textbf{Case 1: Standard Gradient Descent (GD).} \\
    Under GD, $\Delta \bm{\theta} = -\eta \nabla_{\bm{\theta}} \ell$. The alignment yields the squared $\ell_2$-norm:
    \begin{equation}
        \langle \Delta \bm{\theta}, -\nabla_{\bm{\theta}} \ell \rangle = \langle -\eta \nabla_{\bm{\theta}} \ell, -\nabla_{\bm{\theta}} \ell \rangle = \eta \|\nabla_{\bm{\theta}} \ell\|_2^2 \ge 0.
    \end{equation}

    \item \textbf{Case 2: Adam (Sign Gradient).} \\
    To isolate the geometric effect of Adam's adaptive denominator, we abstract its update rule as scaled sign gradient descent, $\Delta \bm{\theta} = -\eta \operatorname{sign}(\nabla_{\bm{\theta}} \ell)$, applied element-wise. Since $x \cdot \operatorname{sign}(x) = |x|$, the alignment resolves to the $\ell_1$-norm:
    \begin{equation}
        \langle \Delta \bm{\theta}, -\nabla_{\bm{\theta}} \ell \rangle = \eta \sum_{i} \operatorname{sign}([\nabla_{\bm{\theta}} \ell]_i) [\nabla_{\bm{\theta}} \ell]_i = \eta \|\nabla_{\bm{\theta}} \ell\|_1 \ge 0.
    \end{equation}
    
    \item \textbf{Case 3: Muon (Orthogonalized Gradient).} \\
    Muon operates on the matrix parameters directly. Let $X \in \{A, B\}$ denote a parameter block with gradient $G_X = \nabla_X \ell$. Using the compact SVD, $G_X = U_X \Sigma_X V_X^\top$. Muon drops the singular values to yield the orthogonalized update $\Delta X = -\eta U_X V_X^\top$. 
    The total parameter alignment is the sum of alignments over the blocks. Using the trace inner product $\langle A, B \rangle = \operatorname{Tr}(A^\top B)$ and its cyclic permutation property:
    \begin{align}
        \langle \Delta X, -G_X \rangle &= \operatorname{Tr}\big( (-\eta U_X V_X^\top)^\top (-U_X \Sigma_X V_X^\top) \big) \\
        &= \eta \operatorname{Tr}\big( V_X U_X^\top U_X \Sigma_X V_X^\top \big) \\
        &= \eta \operatorname{Tr}\big( \Sigma_X V_X^\top V_X \big) = \eta \operatorname{Tr}(\Sigma_X).
    \end{align}
    Since $\operatorname{Tr}(\Sigma_X)$ is the sum of singular values, it exactly equals the nuclear norm $\|G_X\|_*$. Summing over the adapter blocks yields:
    \begin{equation}
        \langle \Delta \bm{\theta}, -\nabla_{\bm{\theta}} \ell \rangle = \eta (\|G_A\|_* + \|G_B\|_*) \ge 0.
    \end{equation}
\end{itemize}
In all cases, the update strictly aligns with the negative gradient, guaranteeing that the signal term mimics a valid FFT descent direction measured in different geometric norms.

\textit{2. Positive semi-definiteness of $\mathcal{H}_{\text{Signal}}$:} \\
Recall from the decomposition that $\mathcal{H}_{\text{Signal}} = J(\bm{\theta})^\top \mathbf{H}_\ell J(\bm{\theta})$. Assume the task landscape is locally convex, meaning the task Hessian is positive semi-definite ($\mathbf{H}_\ell \succeq 0$). For any arbitrary non-zero vector $\bm{v} \in \mathbb{R}^p$, we have:
\begin{equation}
    \bm{v}^\top \mathcal{H}_{\text{Signal}} \bm{v} = \bm{v}^\top (J(\bm{\theta})^\top \mathbf{H}_\ell J(\bm{\theta})) \bm{v} = (J(\bm{\theta})\bm{v})^\top \mathbf{H}_\ell (J(\bm{\theta})\bm{v}).
\end{equation}
Let $\bm{u} \coloneqq J(\bm{\theta})\bm{v} \in \mathbb{R}^D$. Since $\mathbf{H}_\ell \succeq 0$, it follows that $\bm{u}^\top \mathbf{H}_\ell \bm{u} \ge 0$ for all $\bm{u}$. Therefore, $\bm{v}^\top \mathcal{H}_{\text{Signal}} \bm{v} \ge 0$, which proves $\mathcal{H}_{\text{Signal}} \succeq 0$.

\textbf{Part II: Properties of the Drift Terms.}

\textit{1. Indefinite saddle structure of $\mathcal{H}_{\text{Drift}}$:} \\
To prove that the structural drift introduces a saddle point, we analyze the block structure of $\mathcal{H}_{\text{Drift}}$. 
Since the mapping $W(\theta) = W_0 + \frac{\alpha}{r}BA$ is strictly bilinear, all intra-variable second derivatives (\textit{e.g.}, with respect to $A$ or $B$ alone) vanish identically. The parameter Hessian of any scalar entry $w_k$ thus consists solely of cross-derivatives, yielding a strictly block off-diagonal form. The aggregated drift Hessian naturally inherits this structure:
\begin{equation}
    \nabla^2_{\bm{\theta}} w_k = \frac{\alpha}{r} \begin{bmatrix} \mathbf{0} & \mathbf{C}_k \\ \mathbf{C}_k^\top & \mathbf{0} \end{bmatrix} \quad \implies \quad \mathcal{H}_{\text{Drift}} = \sum_{k=1}^D g_k \nabla^2_{\bm{\theta}} w_k = \frac{\alpha}{r} \begin{bmatrix} \mathbf{0} & \mathbf{C} \\ \mathbf{C}^\top & \mathbf{0} \end{bmatrix},
\end{equation}
where $\mathbf{C} = \sum_{k=1}^D g_k \mathbf{C}_k$ aggregates the gradient-weighted cross-derivatives.

By the spectral properties of symmetric block off-diagonal matrices, their non-zero eigenvalues strictly appear in symmetric pairs ($\pm \lambda$). This characterizes $\mathcal{H}_{\text{Drift}}$ as an \textit{indefinite matrix}, establishing that the bilinear parameterization inherently injects a disruptive saddle-point geometry into the landscape, independent of the task curvature $\mathbf{H}_\ell$.

\textit{2. Uncontrolled force of $\Delta \bm{w}_{\text{Drift}}$:} \\
To demonstrate that the structural drift can actively degrade optimization, we examine its alignment with the negative task gradient ($-\bm{g}$) under standard gradient descent. 
Recall from Proposition~\ref{prop:sd_decomp} that the weight-space drift is defined via the structural Hessian tensor as $\Delta \bm{w}_{\text{Drift}} = \frac{1}{2}\nabla^2_{\bm{\theta}} \bm{w}[\Delta \bm{\theta}, \Delta \bm{\theta}]$. Under GD with learning rate $\eta > 0$, the parameter update is $\Delta \bm{\theta} = -\eta \nabla_{\bm{\theta}} \ell = -\eta J(\theta)^\top \bm{g}$.
Evaluating the inner product between the drift step and the ideal descent direction yields:
\begin{equation}
\label{eq:drift_inner_product}
    \langle \Delta \bm{w}_{\text{Drift}}, -\bm{g} \rangle = -\frac{1}{2} \bm{g}^\top \Big( \nabla^2_{\bm{\theta}} \bm{w}[\Delta \bm{\theta}, \Delta \bm{\theta}] \Big) = -\frac{1}{2} \Delta \bm{\theta}^\top \left( \sum_{k=1}^{D} g_k \nabla^2_{\bm{\theta}} w_k \right) \Delta \bm{\theta}.
\end{equation}
Remarkably, the term inside the parenthesis is exactly the parameter-space drift Hessian defined in Proposition~\ref{prop:sd_decomp}. Thus, the inner product simplifies to a quadratic form:
\begin{equation}
    \langle \Delta \bm{w}_{\text{Drift}}, -\bm{g} \rangle = -\frac{1}{2} \Delta \bm{\theta}^\top \mathcal{H}_{\text{Drift}} \Delta \bm{\theta}.
\end{equation}
Because $\mathcal{H}_{\text{Drift}}$ is a block off-diagonal indefinite matrix, this quadratic form is not guaranteed to be non-negative, meaning the drift can easily oppose the descent direction.

\textbf{Counterexample:} Consider a scalar network ($D=1$, rank $r=1$) with $\alpha=1$. The effective weight is $w = ab$, parameterized by $\bm{\theta} = [a, b]^\top$. 
Assume the current parameters are $\bm{\theta} = [1, 1]^\top$ and the task gradient is $g = 1$. The Jacobian is $J(\theta) = \nabla_{\bm{\theta}} w = [b, a] = [1, 1]$, and the parameter update is $\Delta \bm{\theta} = -\eta J(\theta)^\top g = [-\eta, -\eta]^\top$. 
The drift Hessian evaluates to $\mathcal{H}_{\text{Drift}} = g \nabla^2_{\bm{\theta}} w = 1 \cdot \begin{bmatrix} 0 & 1 \\ 1 & 0 \end{bmatrix}$. 
Substituting these into Eq.~\eqref{eq:drift_inner_product}, the inner product is:
\begin{equation}
    \langle \Delta \bm{w}_{\text{Drift}}, -g \rangle = -\frac{1}{2} \begin{bmatrix} -\eta & -\eta \end{bmatrix} \begin{bmatrix} 0 & 1 \\ 1 & 0 \end{bmatrix} \begin{bmatrix} -\eta \\ -\eta \end{bmatrix} = -\frac{1}{2} (2\eta^2) = -\eta^2 < 0.
\end{equation}
\end{proof}

\subsection{Proof for Proposition~\ref{prop:suppression}}
\begin{proposition*}[Spectral Suppression]
Under $B=0$ and $A_{i,j} \sim \mathcal{N}(0, \sigma^2_{A})$, the expected signal curvature captured by LoRA satisfies $\mathbb{E}[\operatorname{Tr}(\mathcal{H}_{\text{Signal}})] = \alpha^2\rho \operatorname{Tr}(\mathbf{H}_{\ell})$, where $\rho = \sigma^2_{A}/r$.
\end{proposition*}
\begin{proof}
Using the property of vectorization $\operatorname{vec}(XYZ) = (Z^\top \otimes X)\operatorname{vec}(Y)$, we can express the LoRA mapping in the vectorized space as:
\begin{equation}\label{eq:vec_lora}
    \bm{w} = \operatorname{vec}(W) + \frac{\alpha}{r}(A^\top \otimes I_{d_{out}})\operatorname{vec}(B),
\end{equation}
where $I_{d_{out}}$ is the identity matrix of dimension $d_{out}$.

Let $J = \nabla_{\bm{\theta}} \bm{w} \in \mathbb{R}^{D \times p}$ denote the Jacobian matrix of this mapping, which can be partitioned into blocks corresponding to $A$ and $B$, such that $J = [J_A, J_B]$. At initialization, since $B = 0$, the partial derivative with respect to $A$ vanishes completely, yielding $J_A = 0$. The partial derivative with respect to $B$ is given by:\begin{equation}
    J_B = \frac{\alpha}{r} (A^\top \otimes I_{d_{out}}).
\end{equation}
The signal curvature projected into the parameter space is defined by $\mathcal{H}_{\text{Signal}} = J^\top \mathbf{H}_{\ell} J$. Taking the trace and applying its cyclic property $\operatorname{Tr}(XYZ) = \operatorname{Tr}(ZXY)$, we have:
\begin{equation}
    \operatorname{Tr}(\mathcal{H}_{\text{Signal}}) = \operatorname{Tr}(J^\top \mathbf{H}_{\ell} J) = \operatorname{Tr}(\mathbf{H}_{\ell} J J^\top) = \operatorname{Tr}(\mathbf{H}_{\ell} J_B J_B^\top).
\end{equation}
We now evaluate the expectation of the outer product $J_B J_B^\top$ over the random Gaussian initialization of $A$. Utilizing the mixed-product property of the Kronecker product, $(X \otimes Y)(U \otimes V) = (XU \otimes YV)$, we obtain:
\begin{equation}
    J_B J_B^\top = \left(\frac{\alpha}{r}\right)^2 (A^\top A \otimes I_{d_{out}}).
\end{equation}
By assumption, the entries of $A \in \mathbb{R}^{r \times d_{in}}$ are drawn i.i.d. from $\mathcal{N}(0, \sigma^2_{A})$. Thus, the expectation of its Gram matrix is $\mathbb{E}[A^\top A] = r\sigma^2_{A} I_{d_{in}}$. Substituting this into the expectation yields:
\begin{equation}
    \mathbb{E}[J_B J_B^\top] = \frac{\alpha^2}{r^2} (\mathbb{E}[A^\top A] \otimes I_{d_{out}}) = \frac{\alpha^2}{r^2} (r\sigma^2_{A} I_{d_{in}} \otimes I_{d_{out}}) = \frac{\alpha^2 \sigma^2_{A}}{r} I_D,
\end{equation}
where $I_D$ is the identity matrix of dimension $D = d_{in}d_{out}$.Finally, by the linearity of the trace and expectation operators, we conclude:
\begin{equation}
    \mathbb{E}[\operatorname{Tr}(\mathcal{H}_{\text{Signal}})] = \operatorname{Tr}\left(\mathbf{H}_{\ell} \mathbb{E}[J_B J_B^\top]\right) = \operatorname{Tr}\left(\mathbf{H}_{\ell} \frac{\alpha^2 \sigma^2_{A}}{r} I_D\right) = \frac{\alpha^2 \sigma^2_{A}}{r} \operatorname{Tr}(\mathbf{H}_{\ell}).
\end{equation}

Defining the variance-to-rank ratio as $\rho = \sigma^2_{A}/r$, we arrive at the exact analytical expression:
\begin{equation}
    \mathbb{E}[\operatorname{Tr}(\mathcal{H}_{\text{Signal}})] = \alpha^2\rho \operatorname{Tr}(\mathbf{H}_{\ell}).
\end{equation}
\end{proof}

\begin{lemma}[Hessian Spectral Bound and Stability]
\label{lemma:stability_bound}
Consider a second-order differentiable loss function $\ell(\theta)$ with a local Hessian $\mathcal{H} = \nabla^2 \ell(\theta)$. If the gradient $\nabla \ell(\theta)$ is $L$-Lipschitz continuous, the stability of gradient descent with a constant learning rate $\eta$ is governed by the maximum eigenvalue of the Hessian, denoted as $\lambda_{\max}(\mathcal{H})$. Specifically, the optimization remains stable if and only if:
\begin{equation}
    \eta \le \frac{2}{L} = \frac{2}{\lambda_{\max}(\mathcal{H})}.
\end{equation}
Furthermore, a reduction in $\lambda_{\max}(\mathcal{H})$ via spectral suppression increases the upper bound of the stable learning rate, effectively expanding the permissible hyperparameter space.
\end{lemma}

\begin{proof}
By the descent lemma for $L$-smooth functions \cite{nesterov2013introductory}, for any $\theta_t$ and $\theta_{t+1} = \theta_t - \eta \nabla \ell (\theta_t)$, we have:
\begin{equation}
    \ell(\theta_{t+1}) - \ell(\theta_t) \le - \eta \|\nabla \ell(\theta_t)\|^2 + \frac{L \eta^2}{2} \|\nabla \ell(\theta_t)\|^2 = - \eta \left( 1 - \frac{L\eta}{2} \right) \|\nabla \ell(\theta_t)\|^2.
\end{equation}
To ensure the loss is non-increasing ($\ell(\theta_{t+1}) \le \ell(\theta_t)$), we require $1 - \frac{L\eta}{2} \ge 0$, which yields $\eta \le 2/L$. In the quadratic approximation of the local optimization landscape, the Lipschitz constant $L$ is equivalent to the spectral norm of the Hessian $\|\mathcal{H}\|_2$, which corresponds to its maximum eigenvalue $\lambda_{\max}(\mathcal{H})$. Thus, $\eta \le 2/\lambda_{\max}(\mathcal{H})$.
\end{proof}

\subsection{Proof for Proposition~\ref{prop:asymmetric_scaling}}
\begin{proposition*}[Asymmetric Scaling]
The components of the landscape and the weight update scale asymmetrically with respect to $\alpha$ and $\eta$. 
For the landscape: $\mathcal{H}_{\text{Signal}} = \Theta(\alpha^2)$ while $\mathcal{H}_{\text{Drift}} = \Theta(\alpha)$. 
For the update under adaptive optimizers like Adam: $\Delta \bm{w}_{\text{Signal}} = \Theta(\alpha\eta)$ while $\Delta \bm{w}_{\text{Drift}} = \Theta(\alpha\eta^2)$.
\end{proposition*}

\begin{proof}
The proof proceeds by first establishing the scaling properties of the mapping derivatives with respect to the hyperparameter $\alpha$, and subsequently analyzing their propagation into the landscape and the update dynamics.

\paragraph{Step 1: Scaling of the Jacobian and Structural Tensor.}
Based on Eq.~\eqref{eq:vec_lora}, the Jacobian $J(\bm{\theta}) = \nabla_{\bm{\theta}} \bm{w} \in \mathbb{R}^{D \times p}$ can be partitioned into blocks corresponding to the partial derivatives with respect to $A$ and $B$:
\begin{equation}
    J(\bm{\theta}) = \left[ \frac{\partial \bm{w}}{\partial A}, \frac{\partial \bm{w}}{\partial B} \right] = \frac{\alpha}{r} \left[ (I_{d_{in}} \otimes B), (A^\top \otimes I_{d_{out}}) \right].
\end{equation}
Because the coefficient $\frac{\alpha}{r}$ is factored out of the partial derivatives, the Jacobian is directly proportional to $\alpha$. Thus, its magnitude scales tightly as:
\begin{equation}
    \|J(\bm{\theta})\| = \Theta(\alpha).
\end{equation}
Similarly, the structural Hessian tensor $\nabla^2_{\bm{\theta}} \bm{w}$ captures the second-order derivatives of the mapping. Because the mapping is bilinear in $A$ and $B$, the only non-zero second derivatives are the cross-derivatives between $A$ and $B$. Differentiating $J(\bm{\theta})$ again yields a tensor that maintains the exact linear dependence on $\alpha$:
\begin{equation}
    \|\nabla^2_{\bm{\theta}} \bm{w}\| = \Theta(\alpha).
\end{equation}

\paragraph{Step 2: Scaling of the Landscape Components.}
We now substitute the scaling rules of the derivatives into the decomposition defined in Proposition~\ref{prop:sd_decomp}. Assuming the full task gradient $\bm{g} = \nabla_{\bm{w}}\ell$ and full task Hessian $\mathbf{H}_\ell = \nabla^2_{\bm{w}}\ell$ are locally independent of the adapter parameterization, the Signal Hessian scales quadratically due to the outer product of the Jacobian:
\begin{equation}
    \|\mathcal{H}_{\text{Signal}}\| = \|J(\bm{\theta})^\top \mathbf{H}_\ell J(\bm{\theta})\| = \Theta(\alpha) \times \Theta(\alpha) = \Theta(\alpha^2).
\end{equation}
In contrast, the Drift Hessian is a linear combination of the structural tensor slices weighted by the gradient entries $g_k$. Thus, it inherits the linear scaling of the structural tensor:
\begin{equation}
    \|\mathcal{H}_{\text{Drift}}\| = \Big\|\sum_{k=1}^{D} g_k \nabla^2_{\bm{\theta}} w_k\Big\| = \Theta(\alpha).
\end{equation}

\paragraph{Step 3: Update Scaling under Adaptive Optimizers.}
A key characteristic of modern adaptive optimizers, such as Adam, is their mechanism to decouple the update magnitude from the raw gradient scale. This process effectively normalizes the gradient magnitude, ensuring that the parameter update $\Delta \bm{\theta}$ is primarily governed by the learning rate $\eta$, rather than the local curvature or gradient variance. Taking Adam as a representative case, the update step is formulated as:
\begin{equation}
    \Delta \bm{\theta} = -\eta \frac{\hat{\bm{m}}}{\sqrt{\hat{\bm{v}}} + \epsilon},
\end{equation}
where $\hat{\bm{m}}$ and $\hat{\bm{v}}$ denote the bias-corrected first and second moment estimates of the gradients, respectively. The denominator term $\sqrt{\hat{\bm{v}}} + \epsilon$ serves as the normalization factor that stabilizes the update step size. Consequently, the magnitude of the update remains consistent at $\|\Delta \bm{\theta}\| = \Theta(\eta)$, irrespective of the absolute norm of the gradient $\nabla_{\bm{\theta}} \ell$.

By applying this $\eta$-invariant step size back to the weight update decomposition in Proposition~\ref{prop:sd_decomp}, we observe an asymmetric scaling behavior between the signal and drift components:
\begin{itemize}[leftmargin=*, topsep=0.3em]
    \item \textbf{Signal Update:} As a first-order linear transformation of the parameter update, the task-aligned signal scales linearly with both $\alpha$ and $\eta$:
    \begin{equation}
        \|\Delta \bm{w}_{\text{Signal}}\| = \|J(\bm{\theta})\Delta \bm{\theta}\| = \Theta(\alpha) \times \Theta(\eta) = \Theta(\alpha\eta).
    \end{equation}
    \item \textbf{Drift Update:} Conversely, the structural drift is defined by the bilinear quadratic form of $\Delta \bm{\theta}$. When the structural tensor is applied to the normalized updates, the $\eta$ scaling is squared while $\alpha$ remains linear:
    \begin{equation}
        \|\Delta \bm{w}_{\text{Drift}}\| = \Big\|\frac{1}{2}\nabla^2_{\bm{\theta}} \bm{w}[\Delta \bm{\theta}, \Delta \bm{\theta}]\Big\| = \Theta(\alpha) \times \Theta(\eta^2) = \Theta(\alpha\eta^2).
    \end{equation}
\end{itemize}
\end{proof}

\section{Experimental Details}~\label{appendix:Implementation}
We summarize the core configurations and hyperparameters across all six evaluation domains in Table~\ref{tab:master_experimental_setup}. Detailed learning rate schedules and environmental configurations strictly follow the referenced official codebases. Subsequent subsections provide granular details, including datasets and evaluation metrics, for each specific task.

\begin{table*}[t]
\centering
\caption{Comprehensive summary of experimental settings across all evaluation domains. The estimated wall-clock time (Est. Time) highlights our evaluation across diverse training horizons, ranging from rapid adaptation to multi-day post-training scenarios.}
\label{tab:master_experimental_setup}
\footnotesize
\renewcommand{\arraystretch}{1.2}
\setlength{\tabcolsep}{3pt}
\begin{tabular*}{\textwidth}{@{\extracolsep{\fill}}lccccccc}
\toprule
\textbf{Task \& Source} & \textbf{Backbone Model} & \textbf{Rank ($r$)} & \textbf{$\alpha$ Strategy} & \textbf{LR ($\eta$)} & \textbf{Batch Size} & \textbf{Duration} & \textbf{Est. Time} \\
\midrule
NLU\footnotemark[1]  & DeBERTa-v3-base-184M & 8 & Variant II  & $1 \times 10^{-4}$ & 32 & 3--5 Epochs & Mins $\sim$ Hours \\
NLG\footnotemark[2] & Llama-2-7B & 16 / 128 & Variant II & $2 \times 10^{-5}$ & 128 & 1 Epoch & Mins $\sim$ Hours \\
DreamBooth\footnotemark[3] & FLUX.1-12B & 8 & Variant II & $1 \times 10^{-4}$ & 1 & 1,000 Steps & Mins $\sim$ Hours \\
Multimodal\footnotemark[4] & Qwen 2-VL (2B/7B) & 16 & Variant I  & $2 \times 10^{-5}$ & 1,024 & 2,000 Steps & 3-4 Days \\
Reasoning SFT\footnotemark[5] & Qwen 2.5-Math-7B & 64 / 256 & Variant I & $4 \times 10^{-5}$ & 128 & 1 Epoch & 2--3 Days \\
Reasoning RL\footnotemark[6] & DeepSeek-R1-Distill (2B/7B) & 64 & Variant I & $1 \times 10^{-6}$ & 128 & 1 Epoch & 2--3 Days \\
\bottomrule
\end{tabular*}
\end{table*}
\footnotetext[1]{\url{https://github.com/huggingface/transformers/tree/main/examples/pytorch/text-classification}}
\footnotetext[2]{\url{https://github.com/MuLabPKU/PiSSA}}
\footnotetext[3]{\url{https://github.com/huggingface/peft/tree/main/examples/lora_dreambooth}}
\footnotetext[4]{\url{https://github.com/TIGER-AI-Lab/VLM2Vec/tree/v1.0}}
\footnotetext[5]{\url{https://github.com/huggingface/open-r1}}
\footnotetext[6]{\url{https://github.com/MikaStars39/PeRL}}

\begin{table*}[t]
\centering
\caption{Comparison on NLG tasks using Llama 2-7B. Performance is reported as mean $\pm$ std over three independent runs. Bold and underlined values denote the best and second-best results within each experimental group.}
\label{tab:nlg_performance_final_std}
\small
\renewcommand{\arraystretch}{0.85}
\setlength{\tabcolsep}{5pt}
\begin{tabular*}{\textwidth}{@{\extracolsep{\fill}}clccccc}
\toprule
\textbf{Rank ($r$)} & \textbf{Method} & \textbf{GSM8K} & \textbf{MATH} & \textbf{HumanEval} & \textbf{MBPP} & \textbf{Commonsense} \\
\midrule
N/A & Full Fine-Tuning & 60.34 \text{\scriptsize$\pm$ 1.32} & 11.74 \text{\scriptsize$\pm$ 0.63} & 32.30 \text{\scriptsize$\pm$ 1.26} & 39.27 \text{\scriptsize$\pm$ 1.01} & 79.20 \text{\scriptsize$\pm$ 1.20} \\
\midrule
\multirow{6}{*}{16} 
& LoRA    & 31.51 \text{\scriptsize$\pm$ 0.31} & 4.16 \text{\scriptsize$\pm$ 0.27} & 15.98 \text{\scriptsize$\pm$ 0.20} & \underline{28.65} \text{\scriptsize$\pm$ 0.47} & 66.56 \text{\scriptsize$\pm$ 1.21} \\
& RsLoRA  & 39.04 \text{\scriptsize$\pm$ 0.53} & 4.94 \text{\scriptsize$\pm$ 0.40} & 18.85 \text{\scriptsize$\pm$ 0.66} & 28.10 \text{\scriptsize$\pm$ 0.64} & 73.24 \text{\scriptsize$\pm$ 0.84} \\
& LoRA+   & 31.69 \text{\scriptsize$\pm$ 0.64} & 3.98 \text{\scriptsize$\pm$ 0.38} & 18.54 \text{\scriptsize$\pm$ 0.52} & 28.00 \text{\scriptsize$\pm$ 0.81} & 72.19 \text{\scriptsize$\pm$ 1.43} \\
& PiSSA   & 37.68 \text{\scriptsize$\pm$ 0.45} & 5.16 \text{\scriptsize$\pm$ 0.41} & 18.37 \text{\scriptsize$\pm$ 0.49} & 28.62 \text{\scriptsize$\pm$ 0.68} & 73.72 \text{\scriptsize$\pm$ 1.05} \\
& LoRAM   & \underline{40.32} \text{\scriptsize$\pm$ 0.43} & \underline{5.30} \text{\scriptsize$\pm$ 0.42} & \textbf{18.92} \text{\scriptsize$\pm$ 0.55} & \textbf{28.83} \text{\scriptsize$\pm$ 0.63} & \underline{75.19} \text{\scriptsize$\pm$ 1.10} \\
& LoRA-$\alpha$ & \textbf{47.61} \text{\scriptsize$\pm$ 0.57} & \textbf{6.84} \text{\scriptsize$\pm$ 0.37} & \underline{18.90} \text{\scriptsize$\pm$ 0.24} & 28.30 \text{\scriptsize$\pm$ 0.70} & \textbf{77.42} \text{\scriptsize$\pm$ 0.93} \\
\midrule
\multirow{6}{*}{128}
& LoRA    & 40.27 \text{\scriptsize$\pm$ 0.70} & 4.72 \text{\scriptsize$\pm$ 0.43} & 20.11 \text{\scriptsize$\pm$ 0.32} & 28.84 \text{\scriptsize$\pm$ 0.37} & 73.64 \text{\scriptsize$\pm$ 1.13} \\
& RsLoRA  & 50.38 \text{\scriptsize$\pm$ 0.37} & \underline{7.32} \text{\scriptsize$\pm$ 0.28} & 21.32 \text{\scriptsize$\pm$ 0.70} & 30.73 \text{\scriptsize$\pm$ 0.43} & 77.01 \text{\scriptsize$\pm$ 1.17} \\
& LoRA+   & 40.41 \text{\scriptsize$\pm$ 0.67} & 5.28 \text{\scriptsize$\pm$ 0.53} & 20.71 \text{\scriptsize$\pm$ 0.88} & 29.13 \text{\scriptsize$\pm$ 0.78} & \underline{78.19} \text{\scriptsize$\pm$ 1.33} \\
& PiSSA   & \underline{51.48} \text{\scriptsize$\pm$ 0.77} & 7.04 \text{\scriptsize$\pm$ 0.54} & 21.62 \text{\scriptsize$\pm$ 0.48} & 31.07 \text{\scriptsize$\pm$ 0.68} & 77.28 \text{\scriptsize$\pm$ 0.98} \\
& LoRAM   & 51.12 \text{\scriptsize$\pm$ 0.73} & 7.25 \text{\scriptsize$\pm$ 0.68} & \underline{22.03} \text{\scriptsize$\pm$ 0.56} & \underline{31.53} \text{\scriptsize$\pm$ 0.72} & 77.81 \text{\scriptsize$\pm$ 0.96} \\
& LoRA-$\alpha$  & \textbf{53.15} \text{\scriptsize$\pm$ 0.76} & \textbf{8.30} \text{\scriptsize$\pm$ 0.35} & \textbf{23.20} \text{\scriptsize$\pm$ 0.45} & \textbf{33.90} \text{\scriptsize$\pm$ 0.76} & \textbf{78.38} \text{\scriptsize$\pm$ 0.87} \\
\bottomrule
\end{tabular*}
\end{table*}

\begin{table*}[h]
\centering
\caption{Detailed hyperparameter search results on the GLUE benchmark. The table includes a grid search over learning rate ($\eta$) and scaling factor ($\alpha$). The average is computed across all 8 tasks.}
\label{tab:hyperparam_search}
\small
\renewcommand{\arraystretch}{0.95}
\setlength{\tabcolsep}{4.5pt}
\begin{tabular*}{\textwidth}{@{\extracolsep{\fill}}lccccccccc}
\toprule
\textbf{Configuration} & \textbf{MNLI} & \textbf{SST-2} & \textbf{MRPC} & \textbf{CoLA} & \textbf{QNLI} & \textbf{QQP} & \textbf{RTE} & \textbf{STS-B} & \textbf{Avg.} \\ \midrule

\multicolumn{10}{c}{\textit{Learning Rate ($\eta$) Search, fixing $\alpha=r$}} \\ \midrule
$\eta = 1.0\times 10^{-5}$ & 86.99 & 93.00 & 68.38 & 0.00 & 90.02 & 87.68 & 47.29 & 24.06 & 62.18 \\
$\eta = 2.5\times 10^{-5}$ & 88.83 & 94.50 & 68.38 & 58.33 & 91.74 & 88.98 & 47.29 & 75.96 & 76.75 \\
$\eta = 5.0\times 10^{-5}$ & 89.64 & 94.95 & 80.64 & 61.08 & 92.90 & 89.89 & 47.29 & 81.44 & 79.73 \\
$\eta = 7.5\times 10^{-5}$ & 90.10 & 95.53 & 83.33 & 60.82 & 93.43 & 90.39 & 47.29 & 84.76 & 80.71 \\
$\eta = 1.0\times 10^{-4}$ & 90.23 & 95.87 & 84.06 & 63.56 & 93.88 & 90.55 & 50.18 & 87.20 & 81.94 \\
$\eta = 2.5\times 10^{-4}$ & \textbf{90.69} & \textbf{96.21} & 88.97 & 67.26 & 93.94 & 91.48 & 76.17 & 88.45 & 86.65 \\
$\eta = 5.0\times 10^{-4}$ & 90.46 & 95.07 & 89.46 & 68.48 & \textbf{94.05} & \textbf{91.76} & 82.31 & 89.74 & 87.67 \\
$\eta = 7.5\times 10^{-4}$ & 89.94 & 95.99 & \textbf{91.05} & 68.06 & 93.87 & 91.07 & \textbf{84.47} & 88.54 & \textbf{87.87} \\
$\eta = 1.0\times 10^{-3}$ & 88.03 & 94.10 & 89.20 & \textbf{69.27} & 92.01 & 90.02 & 82.12 & \textbf{90.72} & 86.93 \\ \midrule

\multicolumn{10}{c}{\textit{Scaling Factor ($\alpha$) Search, leveraging LoRA-$\alpha$ and  fixing $\eta=1.0\times10^{-4}$}} \\ \midrule
$\alpha = 0.1$ & 90.68 & 95.87 & 85.78 & 65.29 & 93.98 & 90.95 & 67.15 & 87.85 & 84.69 \\
$\alpha = 0.25$ & 90.65 & \textbf{96.44} & 86.76 & 65.78 & 93.96 & 91.42 & 76.17 & 88.28 & 86.18 \\
$\alpha = 0.5$ & \textbf{90.75} & 96.33 & 87.01 & 67.49 & \textbf{94.05} & 91.75 & 79.78 & 89.70 & 87.11 \\
$\alpha = 0.75$ & 90.45 & 95.99 & 88.73 & 68.23 & 93.85 & \textbf{92.06} & 83.39 & 89.79 & 87.81 \\
$\alpha = 1.0$  & 90.43 & 95.75 & {89.95} & \textbf{70.89} & 94.03 & 91.74 & 80.86 & 90.12 & 87.97 \\
$\alpha = 2.5$ & 89.65 & 95.07 & \textbf{90.44} & 67.76 & 93.70 & 91.37 & 85.92 & \textbf{91.11} & \textbf{88.12} \\
$\alpha = 5.0$ & 89.05 & 93.81 & {89.95} & 67.32 & 91.98 & 63.18 & \textbf{87.00} & 90.26 & 84.07 \\
$\alpha = 7.5$ & 31.82 & 50.92 & 89.46 & 67.04 & 49.46 & 63.18 & 47.29 & 91.03 & 61.28 \\
$\alpha = 10.0$ & 32.74 & 50.92 & 88.24 & 64.64 & 49.46 & 63.18 & 47.29 & 1.39 & 49.73 \\
\midrule
\multicolumn{10}{c}{\textit{Aggregating above  results}} \\ 
\midrule
LoRA-$\eta^{*}$ & 90.69 & 96.21 & 91.05 & 69.27 & \textbf{94.05} & 91.76 & 84.47 & 90.72 & 88.53 \\
{LoRA-$\alpha^{*}$} & \textbf{90.75} & \textbf{96.44} & \textbf{90.44} & \textbf{70.89} & \textbf{94.05} & \textbf{92.06} & \textbf{87.00} & \textbf{91.11} & \textbf{89.03} \\
\bottomrule
\end{tabular*}
\end{table*}

\begin{table*}[t]
\centering
\caption{Detailed Hyperparameter Search Results on NLG tasks (Llama2-7B). The table includes a grid search over learning rate ($\eta$) and scaling factor ($\alpha$). The average is computed across the 5 tasks.}
\label{tab:nlg_hyperparam_search}
\small
\renewcommand{\arraystretch}{0.95}
\setlength{\tabcolsep}{6pt}
\begin{tabular*}{\textwidth}{@{\extracolsep{\fill}}lcccccc}
\toprule
\textbf{Configuration} & \textbf{GSM8K} & \textbf{MATH} & \textbf{HumanEval} & \textbf{MBPP} & \textbf{Commonsense} & \textbf{Avg.} \\ \midrule

\multicolumn{7}{c}{\textit{Learning Rate ($\eta$) Search at Rank $r=16$, fixing $\alpha=16$}} \\ \midrule
$\eta = 2.0\times 10^{-5}$ & 31.51 & 4.16 & 15.98 & {28.65} & 66.56 & 29.37 \\
$\eta = 1.0\times 10^{-4}$ & 42.61 & 5.90 & 17.70 & 29.40 & 74.66 & 34.05 \\
$\eta = 2.5\times 10^{-4}$ & 50.27 & 6.46 & 20.10 & 30.70 & 76.32 & 36.77 \\
$\eta = 5.0\times 10^{-4}$ & \textbf{54.28} & 8.44 & \textbf{24.40} & 33.60 & \textbf{78.16} & \textbf{39.78} \\
$\eta = 7.5\times 10^{-4}$ & 52.86 & 9.14 & 23.80 & 34.40 & \textbf{78.16} & 39.67 \\
$\eta = 1.0\times 10^{-3}$ & 51.43 & \textbf{9.50} & 22.20 & \textbf{34.70} & 78.12 & {39.19} \\ \midrule

\multicolumn{7}{c}{\textit{Scaling Factor ($\alpha$) Search at Rank $r=16$, leveraging LoRA-$\alpha$} and fixing $\eta = 2.0\times 10^{-5}$} \\ \midrule
$\alpha = 1.0$ & 47.61 & 6.84 & 18.90 & 28.30 & 77.42 & 35.81 \\
$\alpha = 2.5$ & 51.71 & 7.58 & 21.30 & 32.50 & \textbf{78.19} & 38.26 \\
$\alpha = 5.0$ & 54.51 & 9.18 & 26.20 & 34.70 & 78.16 & \textbf{40.55} \\
$\alpha = 7.5$ & \textbf{55.34} & 9.38 & 23.80 & \textbf{35.40} & 77.96 & 40.38 \\
$\alpha = 10.0$ & 52.59 & \textbf{9.84} & \textbf{28.70} & 32.50 & 76.30 & {39.99} \\ \midrule

\multicolumn{7}{c}{\textit{Learning Rate ($\eta$) Search at Rank $r=128$, fixing $\alpha=128$}} \\ \midrule
$\eta = 2.0\times 10^{-5}$ & 40.27 & 4.72 & 20.11 & 28.84 & 73.64 & 33.52 \\
$\eta = 1.0\times 10^{-4}$ & 52.39 & 7.74 & 21.30 & 33.90 & 77.79 & 38.62 \\
$\eta = 2.5\times 10^{-4}$ & 56.71 & 9.80 & 25.60 & 35.20 & \textbf{78.67} & \textbf{41.20} \\
$\eta = 5.0\times 10^{-4}$ & \textbf{58.83} & \textbf{10.76} & \textbf{28.70} & \textbf{36.50} & 28.19 & {32.60} \\
$\eta = 7.5\times 10^{-4}$ & 1.29 & 1.38 & 28.37 & 34.10 & 27.46 & 18.52 \\
$\eta = 1.0\times 10^{-3}$ & 0.00 & 0.00 & 28.00 & 33.90 & 27.38 & 17.86 \\ \midrule

\multicolumn{7}{c}{\textit{Scaling Factor ($\alpha$) Search at Rank $r=128$, leveraging LoRA-$\alpha$  and fixing $\eta = 2.0\times 10^{-5}$}} \\ \midrule
$\alpha = 1.0$ & 53.15 & 8.30 & 23.20 & 33.90 & 78.38 & 39.39 \\
$\alpha = 2.5$ & 57.54 & 9.78 & 27.40 & 36.80 & \textbf{79.26} & 42.16 \\
$\alpha = 5.0$ & 59.28 & \textbf{11.74} & 28.70 & 36.00 & 76.92 & \textbf{42.53} \\
$\alpha = 7.5$ & \textbf{59.97} & 11.08 & 29.30 & \textbf{39.90} & 72.12 & 42.47 \\
$\alpha = 10.0$ & 58.53 & 11.34 & \textbf{33.50} & 35.20 & 73.25 & 42.36 \\ \midrule

\multicolumn{7}{c}{\textit{Aggregating above results}} \\ \midrule
LoRA-$\eta^{*}$ ($r=16$) & 54.28 & 9.50 & 24.40 & 34.70 & 78.16 & 40.21 \\
LoRA-$\alpha^{*}$ ($r=16$) & \textbf{55.34} & \textbf{9.84} & \textbf{28.70} & \textbf{35.40} & \textbf{78.19} & \textbf{41.49} \\ \midrule
LoRA-$\eta^{*}$ ($r=128$) & 58.83 & 10.76 & 28.70 & 36.50 & 78.67 & 42.69 \\
LoRA-$\alpha^{*}$ ($r=128$) & \textbf{59.97} & \textbf{11.74} & \textbf{33.50} & \textbf{39.90} & \textbf{79.26} & \textbf{44.87} \\
\bottomrule
\end{tabular*}
\end{table*}

\subsection{Natural Language Understanding }

\paragraph{Models and Datasets.} 
We evaluate NLU performance using the DeBERTa-v3-base model (184M)~\cite{he2021debertav3} on eight tasks from the GLUE benchmark~\cite{Wang2018GLUEAM}. Following standard protocols, we include MNLI, SST-2, MRPC, CoLA, QNLI, QQP, RTE, and STS-B.

\paragraph{LoRA Configuration.} 
All parameter-efficient methods are implemented with a fixed rank $r=8$. For the proposed LoRA-$\alpha$, we apply the analytic scaling (Variant II) to determine $\alpha_{\text{base}}$. For other baselines, the scaling factor $\alpha$ follows their respective original designs. Adapters are uniformly applied to all linear layers.

\paragraph{Optimization.} 
All models are trained using the AdamW optimizer with a constant learning rate of $1 \times 10^{-4}$, aligned with standard FFT settings. Training is conducted for 3 epochs across most tasks, with the exception of MRPC, which is trained for 5 epochs due to its limited sample size. We use a batch size of 32 and follow the default weight decay settings from the Transformers library.

\paragraph{Evaluation.} 
We report task-specific metrics as per GLUE conventions: matched/mismatched accuracy for MNLI, Matthews correlation for CoLA, Pearson correlation for STS-B, and accuracy for the remaining tasks. Following~\cite{PiSSA,zhang2025primacy}, results are averaged over three independent runs.

\subsection{Natural Language Generation}

\paragraph{Models and Datasets.} 
We evaluate generative performance via supervised fine-tuning of Llama 2-7B~\cite{touvron2023llama} across three distinct domains. Specifically, the model is trained on MetaMathQA~\cite{yu2023metamath} for mathematical reasoning (evaluated on GSM8K~\cite{cobbe2021gsm8k} and MATH~\cite{yu2023metamath}), CodeFeedback~\cite{zheng2024opencodeinterpreter} for code generation (evaluated on HumanEval~\cite{chen2021evaluating} and MBPP~\cite{austin2021program}), and Commonsense170K~\cite{hu2023llmadapters} for commonsense reasoning (reporting averaged accuracy across eight sub-datasets). Following the PiSSA training framework~\cite{PiSSA}, each task strictly utilizes a sampled subset of 100k examples to ensure a balanced and computationally equivalent comparison.

\paragraph{LoRA Configuration.} 
We evaluate performance across two distinct rank regimes: $r \in \{16, 128\}$. To explicitly validate the effectiveness of our theoretically derived scaling rule, LoRA-$\alpha$ strictly utilizes the analytic scaling ({Variant II}) to determine the scaling factor $\alpha$. This setup verifies whether the analytic formulation can scale to LLMs. Adapters are uniformly applied to all linear layers.

\paragraph{Optimization.} 
Training is performed using the AdamW optimizer with a learning rate of $2 \times 10^{-5}$ and a total batch size of 128. All models are optimized for a single epoch, providing a stringent testbed to assess the rapid adaptation efficiency of different scaling strategies. Following~\cite{PiSSA,zhang2025primacy}, results are averaged over three independent runs. Table~\ref{tab:nlg_performance_final_std} augments the results in Table~\ref{tab:nlg_performance_final} by including standard deviations (std) for reference.

\subsection{Detailed Hyperparameter Search Results}
\label{appendix:detailed_hyperparam_search}

To validate the superiority of $\alpha$-scaling over conventional $\eta$-scaling in generalization, we conduct an exhaustive hyperparameter grid search across both NLU and NLG tasks. The detailed, per-task performance breakdowns are provided in Table~\ref{tab:hyperparam_search} and Table~\ref{tab:nlg_hyperparam_search}. 
\textit{Given the massive scale of this hyperparameter search space, all results reported in these tables are based on a single experimental run per configuration.}
\begin{itemize}[leftmargin=*, topsep=0em, itemsep=0em]
    \item \textbf{NLU Grid Search (Table~\ref{tab:hyperparam_search}):} We observe that strictly scaling the learning rate ($\eta$) often leads to suboptimal plateaus or local instability (\textit{e.g.}, initial failure to converge on CoLA). In contrast, scaling the structural multiplier $\alpha$ unlocks higher performance ceilings across almost all datasets, raising the aggregated peak performance (LoRA-$\alpha^\star$) to 89.03, compared to 88.53 for $\eta$-scaling.
    \item \textbf{NLG Grid Search (Table~\ref{tab:nlg_hyperparam_search}):} The structural bottleneck of $\eta$-scaling becomes critically evident, especially at larger ranks ($r=128$). Pushing $\eta$ to extremes triggers catastrophic divergence (\textit{e.g.}, performance collapsing to 0.00 on MATH and GSM8K). Conversely, $\alpha$-scaling preserves optimization stability under the same aggressive scaling regimes, successfully driving the model to deeper local minima and extending the frontier significantly (averaging 44.87 vs. 42.69).
\end{itemize}

Ultimately, aggregating the optimal configurations explicitly demonstrates that amplifying the signal via $\alpha$ maintains optimization advantage.

\subsection{Text-to-Image Synthesis}

\paragraph{Models and Datasets.} 
We evaluate multimodal generation using the Flux.1-12B model~\cite{flux2024} for image customization via the DreamBooth protocol~\cite{ruiz2023dreambooth}. The implementation is based on the official PEFT DreamBooth examples. The model learns novel concepts from small sets of instance images (typically 4--6) to generate high-fidelity, prompt-conditioned images.

\paragraph{LoRA Configuration.} 
  For LoRA-$\alpha$, we apply analytic scaling (Variant II) to calibrate $\alpha_{\text{base}}$ based on the Flux architecture, with a fixed rank of $r=8$. Adapters are applied to all attention layers

\paragraph{Optimization.} 
Models are optimized using AdamW with a constant learning rate of $1 \times 10^{-4}$ and a batch size of 1. Training is conducted for 1,000 steps. This setting provides a challenging environment to test whether $\alpha$-scaling can accelerate concept acquisition without over-fitting.

\paragraph{Evaluation.} 
Following DreamBooth, we use the unique identifier \texttt{"sks"} for concept binding. To test the model's ability to preserve the learned concept under novel contexts, evaluation prompts include category and scene descriptions: \texttt{"A sks <object\_name> in a sunlit forest, vibrant foliage, ultra-realistic photography"} and \texttt{"A sks <object\_name>, centered on a pure white background, ultra-realistic photography"}.

\subsection{Multimodal Representation Learning}

\paragraph{Models and Datasets.} 
We evaluate LoRA-$\alpha$ on the MMEB benchmark~\cite{jiang2025vlmvec}, which converts generative models into dense retrievers. We use Qwen 2-VL (2B and 7B)~\cite{Qwen2-VL} as backbones, implementing the contrastive training pipeline based on the VLM2Vec codebase. The dataset includes 36 subsets; 20 in-distribution (ID) subsets are used for training, while 16 out-of-distribution (OOD) subsets are reserved for zero-shot evaluation.

\paragraph{LoRA Configuration.} 
For LoRA-$\alpha$, we adopt Variant I with $\alpha_{\text{base}} = 256\sqrt{r}$. Adapters are applied to all linear layers with a fixed rank of $r=16$.  

\paragraph{Optimization.} 
Training employs the InfoNCE contrastive loss (temperature $0.02$) and the AdamW optimizer. We use a global batch size of 1,024 and a learning rate of $2 \times 10^{-5}$. Training lasts for 2,000 steps with gradient caching enabled.

\paragraph{Evaluation.} 
Representations are extracted from the final-token hidden state. Performance is measured via Precision@1 averaged across ID and OOD splits. This setup rigorously tests the model's ability to maintain pretrained multimodal alignment while adapting to discriminative objectives.

\subsection{Experimental Details for Reasoning-based SFT}

\paragraph{Models and Datasets.} 
We employ Qwen 2.5-Math-7B~\cite{qwenmath} as the backbone for long-context reasoning SFT. Training is conducted on the Mixture-of-Thoughts dataset (350k samples)~\cite{openr1} utilizing the training recipes provided by Open-R1. 

\paragraph{LoRA Configuration.} 
We apply adapters to all linear layers with ranks $r \in \{64, 256\}$. Furthermore, to support the generation of special tokens, we additionally unfreeze and train the embedding and output layers. LoRA-$\alpha$ uses Variant I scaling. 

\paragraph{Optimization.} 
Models are trained for 1 epoch using the Adam optimizer. We use a peak learning rate of $4.0 \times 10^{-5}$. The effective batch size is 128. We enable gradient checkpointing and set the maximum sequence length to 32,768 tokens.

\paragraph{Evaluation.} 
Following the established protocols of Open-R1, we evaluate  reasoning performance using the Lighteval framework~\cite{lighteval}. Our evaluation spans a comprehensive suite of challenging benchmarks, including AIME 24/25~\cite{aime24, aime25}, MATH-500~\cite{math500}, GPQA Diamond~\cite{rein2024gpqa}, and LiveCodeBench v4~\cite{jain2024livecodebench}. Performance is reported as Pass@1 accuracy, where the number of generation samples per prompt strictly adheres to the Open-R1 configuration to ensure a standardized and fair comparison across different methods.
\subsection{Experimental Details for Reasoning-based RL}

\paragraph{Models and Datasets.} 
We transition to policy optimization using Group Relative Policy Optimization (GRPO)~\cite{guo2025deepseek}. We use DeepSeek-R1-Distill-Qwen (1.5B and 7B) as base models, trained on the DAPO-Math-17k dataset~\cite{yu2025dapo}. We implement the RLVR pipeline based on the PeRL framework~\cite{yin2025evaluating}. This paradigm optimizes reasoning paths using outcome-based rewards.

\paragraph{LoRA Configuration.} 
Adapters are applied to all linear layers with a rank of $r=64$. We utilize Variant I for LoRA-$\alpha$. To accelerate rollouts, we use a co-located vLLM generation engine~\cite{kwon2023efficient} with 0.4 memory utilization.

\paragraph{Optimization.} 
Training is performed with a constant learning rate of $1 \times 10^{-6}$, a KL penalty $\beta=0$, and a total batch size of 128. For each prompt, we generate $G=8$ rollouts. The maximum completion length is set to 16,384 tokens to test long-form stability.

\paragraph{Evaluation.} 
We assess reasoning performance on MATH-500~\cite{math500}, AMC 23~\cite{maa2023amc}, AIME 24/25~\cite{aime24, aime25}, Minerva~\cite{Minerva}, and HMMT~\cite{Matharena}, using Lighteval framework~\cite{lighteval}. Performance is reported as Pass@1 accuracy estimated over multiple samples.

% \newpage
% \input{sec/checklist.tex}

\end{document}